\pdfoutput=1
\documentclass[10pt]{article}

\usepackage[numbers]{natbib}
\usepackage[utf8]{inputenc}
\usepackage[T1]{fontenc}
\usepackage{hyperref}
\usepackage{url}
\usepackage{booktabs}
\usepackage{amsfonts}
\usepackage{microtype}
\usepackage{nicefrac,xfrac}
\usepackage{geometry} 
\usepackage{multirow}
\usepackage{lipsum}
\usepackage{latexsym,dsfont}
\usepackage{graphicx}
\usepackage[font=normalsize]{caption}
\usepackage{subcaption}
\usepackage{amsmath}
\usepackage{mathtools}
\usepackage{amssymb}
\usepackage{bbold}
\usepackage{enumitem}   
\usepackage{algorithm}
\usepackage{algpseudocode}
\usepackage{tcolorbox}
\usepackage{xcolor}

\newcommand*\diff{\mathop{}\!\mathrm{d}}

\usepackage{color}

\newcommand*{\email}[1]{
    \normalsize\href{mailto:#1}{\texttt{#1}}
    }

\usepackage{amsthm}
\newtheorem{theorem}{Theorem}
\newtheorem{remark}{Remark}
\newtheorem{lemma}{Lemma}
\newtheorem{definition}{Definition}
\newtheorem{corollary}{Corollary}
\newtheorem{proposition}{Proposition}

\newtheorem{example}{Example}

\begin{document}
\title{Breaking Fair Binary Classification with Optimal Flipping Attacks}
\date{University of Wisconsin-Madison}
\author{Changhun Jo \\ \email{cjo4@wisc.edu} \and Jy-yong Sohn \\ \email{sohn9@wisc.edu} \and Kangwook Lee \\ \email{kangwook.lee@wisc.edu}}

\maketitle

\begin{abstract}
Minimizing risk with fairness constraints is one of the popular approaches to learning a fair classifier.
Recent works showed that this approach yields an unfair classifier if the training set is corrupted.
In this work, we study the minimum amount of data corruption required for a successful flipping attack.
First, we find lower/upper bounds on this quantity and show that these bounds are tight when the target model is the unique unconstrained risk minimizer.
Second, we propose a computationally efficient data poisoning attack algorithm that can compromise the performance of fair learning algorithms.
\end{abstract}

\section{Introduction}\label{sec:1}
Fairness and robustness are two main requirements for trustworthy artificial intelligence (AI).
According to the fairness principle in~\cite{EC2019Ethics}, AI systems should ensure that individuals and groups are free from unfair bias and discrimination.
In recent years, researchers have proposed various definitions for fair classification~\cite{Feldman_2015,Hardt_2016} and algorithms for learning fair models~\cite{Kamiran2012Preprocessing,Zemel2013FairLearning,Calmon2017Preprocessing,Grover2020FairGM,Jiang2020Identifying,Zafar2017Mistreatment,Zafar2017FC,Roh2020FR-Train,Roh2020FairBatch,Hardt_2016,Abernethy2020Adaptive}.
One popular approach is to solve risk minimization with constraints that capture the desired fairness definition.

While several works theoretically analyzed the risk minimization with fairness constraints~\cite{Agarwal2018Reduction,Donini_2018}, our understanding of its performance on noisy or corrupted data is scarce.
Given that the use of web-scale training data, crawled from the Internet and/or crowdsourced, has become an essential part of machine learning pipeline~\cite{deng2009imagenet,devlin2018bert,brown2020language}, it is of utmost importance to understand how one can learn fair models on data that is potentially corrupted by random or adversarial noise.
To understand the robustness of risk minimization with fairness constraints,~\cite{Chang2020Adversarial} studied the worst-case scenario -- called \emph{data poisoning attacks} -- where adversaries can modify training data to make the model learned on it becomes unusable (either due to low accuracy or bias).
They designed an online gradient descent algorithm, which adds poisoned samples to the training set iteratively.
Their experimental results showed that constrained risk minimization is so unstable under their attack that the models learned by this approach might be even more unfair than the models learned by unconstrained risk minimization.
However, the optimality of the proposed attack algorithm was unknown.

\begin{figure}[t]
\centering
\includegraphics[width=0.5\columnwidth]{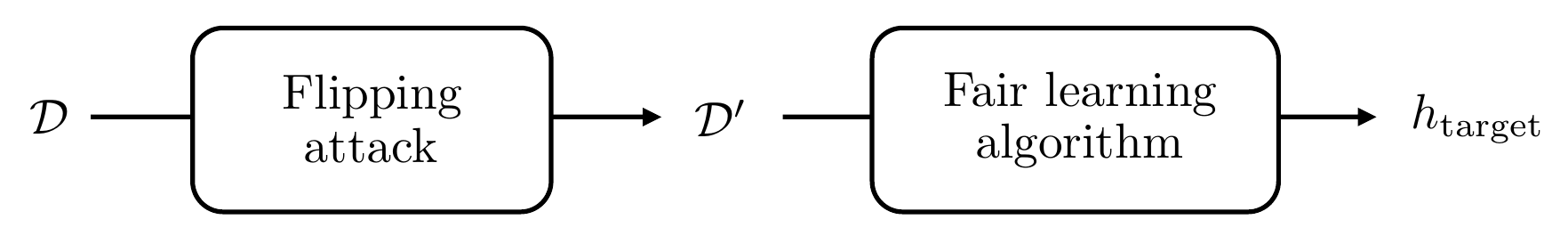}
\caption{A visualization of our framework. Given a target model $h_{\text{target}}$, the attacker wants to find the poisoned distribution $\mathcal{D}'$ via flipping attacks (see Sec.~\ref{sec:3} for a rigorous definition of flipping attacks) on the clean distribution $\mathcal{D}$ so that the fair learning algorithm outputs $h_{\text{target}}$.
Among such flipping attacks, the attacker's goal is to find the optimal flipping attack that minimizes the total variation distance between $\mathcal{D}$ and $\mathcal{D}'$.
See Sec.~\ref{sec:3} for more details.
}\label{fig:1}
\end{figure}

In this work, we study the problem of developing the optimal flipping attack algorithm against risk minimization with fairness constraints.
In particular, we consider a problem setup where an attacker manipulates the data distribution $\mathcal{D}$ such that the model learned on the poisoned data distribution $\mathcal{D}'$ becomes a given target model $h_{\text{target}}$ (see Fig.~\ref{fig:1} for a visualization of our framework).
By formulating this attack problem as a bilevel optimization problem, we provide lower and upper bounds on the minimum amount of data perturbation required for a successful flipping attack.
Furthermore, if the target model is the unique unconstrained risk minimizer (which generally is \emph{unfair}), then our bounds are tight, and our upper bound provides an explicit construction of the optimal flipping attack algorithm.
In other words, when the attacker's goal is to counteract the fairness constraints, our attack algorithm can achieve the goal by perturbing the minimum amount of data.
As a byproduct of our analysis, we also show that, under mild assumptions, there exist infinitely many non-trivial fair models that do not suffer from disparate treatment~\cite{barocas2016big}, which can be of independent theoretical interest.

\section{Related Work}\label{sec:2}
\subsection{Learning Fair Classifiers}
Various metrics have been proposed to measure the fairness of a classification model such as demographic parity~\cite{Feldman_2015}, equalized odds~\cite{Hardt_2016}, and equal opportunity~\cite{Hardt_2016}.
Many methods have been proposed to learn fair classifiers, and they can be grouped in four categories: (1) pre-processing methods~\cite{Kamiran2012Preprocessing,Zemel2013FairLearning,Calmon2017Preprocessing,Grover2020FairGM,Jiang2020Identifying} that preprocess or reweight training data, (2) in-processing methods~\cite{Kamishima2012Fairness,Zafar2017Mistreatment,Zafar2017FC,Yao2017FMC,Agarwal2018Reduction,Zhang2018Mitigating,Cotter2019TwoPlayer,Roh2020FR-Train} that enforce fairness constraints or regularizers during the training period, (3) post-processing methods~\cite{Kamiran2012Decision,Hardt_2016,Pleiss2017FairCalibration,Chzhen2019Leveraging} that manipulate trained models, and (4) adaptive batch selection methods~\cite{Roh2020FairBatch,roh2021sample}.
Several works have also studied fair classification with missing/noisy sensitive attributes~\cite{lamy2019noise,awasthi2020imperfect,Wang2020Robust,mehrotra2021noisy,celis2021noisy,jung2022partial}.

One prominent approach to learning fair classifiers is Fair Empirical Risk Minimization (FERM), an in-processing method, that solves empirical risk minimization with constraints that capture the desired fairness notion.
As fairness constraints are generally non-convex, various relaxations and approximate algorithms have been proposed~\cite{Zafar2017FC,Donini_2018}.
While these algorithms are shown to successfully learn fair classifiers, the robustness to adversarial attack is not fully understood yet.

\subsection{Data Poisoning Attacks and Defenses}
Data poisoning attacks poison the training set to achieve the adversary's goal~\cite{dalvi2004adversarial,Lowd2005GoodWA}, and there are two popular approaches; objective-driven attacks and model-targeted attacks.
The goal of objective-driven attacks~\cite{Biggio2012Poisoning,Xiao2012AdversarialLF,Mei2015Attack,Koh2018Stronger,Chang2020Adversarial,Solans2020Poisoning} is to make the learner output a model satisfying a target property, \emph{e.g.,} low accuracy.
The goal of model-targeted attacks~\cite{Mei2015Attack,Koh2018Stronger,Suya2020Modeltargeted} is to make the learner output a predefined target model.

A few works suggested data poisoning attacks for degrading fairness of learned models.
In~\cite{Solans2020Poisoning}, Solans et al. proposed a gradient-based poisoning attack against ERM to degrade model fairness without significantly degrading accuracy, but theoretical guarantees are missing. 
Recent works proposed online gradient descent algorithms for poisoning attacks against FERM, with respect to various fairness notions~\cite{Chang2020Adversarial,mehrabi2021exacerbating,van2022poisoning}.
In~\citep{Chang2020Adversarial}, Chang et al. proposed an online gradient descent algorithm for poisoning attacks against FERM with a theoretical performance guarantee. They empirically showed that FERM is less robust than ERM in terms of both accuracy and fairness. In~\citep{van2022poisoning}, Van et al. generalized the framework proposed in~\citep{Chang2020Adversarial} and provided an online gradient descent algorithm that can be used for multiple fairness notions. 
All these existing attack methods are categorized as objective-driven attacks aiming at degrading fairness, while we study a model-targeted attack where the attacker’s goal is to make the fair learner output a predefined target model, when trained on the corrupted data. By varying the target model, the attacker in our work can achieve various objectives such as lowering classification accuracy and degrading fairness.
Note that we consider the setting where attackers are able to flip the labels and sensitive attributes of data, inspired by recent works on label-flipping attacks~\cite{Zhao2017LabelCA,Paudice2018LabelSA,Rosenfeld2020RobustLA}.

Several works have theoretically analyzed the behavior of the fairness-aware learner under data poisoning attacks.
The authors of~\cite{celis2021fair} proposed a fair learning algorithm with guaranteed accuracy and fairness, under adversarial perturbation on labels and sensitive attributes. 
The authors of~\cite{konstantinov2021fairness,konstantinov2022impossibility} analyzed how the risk and unfairness of the fair learner change as a function of the fraction of the corrupted data, against the attacker who can perturb features, labels, and sensitive attributes of a random subset of the training set.
Specifically, \cite{konstantinov2021fairness} provided order-optimal upper/lower bounds on the achievable risk and unfairness performances in a PAC learning sense.
Compared with these existing works, the present paper has two key differences in attacker's goal and the attack model. 
First, while \cite{celis2021fair,konstantinov2021fairness,konstantinov2022impossibility} focused on objective-driven attacks where the attacker's goal is to degrade the accuracy/fairness performance, we consider model-targeted attacks and analyze the minimum amount of perturbation required for a fair learner outputting a predefined target model.
Second, given a fixed budget (number of samples) for data poisoning, the attacker considered in~\cite{konstantinov2021fairness,konstantinov2022impossibility} poisons a random subset of the samples, while the attacker in our work can choose which subset to poison.

\section{Problem Formulation}\label{sec:3}
\subsection{Data distribution}
Let $\mathcal{X}$ denote the set of feature vectors, $\mathcal{Y}$ denote the set of labels, and $\mathcal{Z}$ denote the set of sensitive attributes, \emph{e.g.,} gender and race.
We restrict our attention to the case where $\mathcal{X}$ is the $n$-dimensional real space for any natural number $n$, and $\mathcal{Y}$ and $\mathcal{Z}$ are binary, \emph{i.e.,} $\mathcal{X} = \mathbb{R}^n$ and $\mathcal{Y}=\mathcal{Z}=\{0,1\}$.
Let $X, Y,$ and $Z$ be the jointly distributed random variables that take values in $\mathcal{X}, \mathcal{Y},$ and $\mathcal{Z}$, respectively.
Let $\mathcal{D}$ be the joint distribution of $X, Y,$ and $Z$.
Then $\Pr_{\mathcal{D}}(\cdot)$ and $\mathbb{E}_{\mathcal{D}}[\cdot]$ denote the probability\footnote{For ease of presentation, we did not mention the $\sigma$-algebra over which $\Pr_{\mathcal{D}}(\cdot)$ is defined.
When the ambient space is $\mathbb{R}^n$, we consider the Lebesgue $\sigma$-algebra, the collection of all Lebesgue measurable sets.
When the ambient space is a finite set, we use its power set, the collection of all subsets of it.} and expectation over $\mathcal{D}$, respectively.

We use $\mathcal{D}_{X|Y=y,Z=z}$ to denote the distribution of $X$ conditioned on $Y=y,Z=z$ for each $(y,z)\in \mathcal{Y}\times \mathcal{Z}$, and $\mathcal{D}_X$ to denote the marginal distribution of $X$.
For analytical purposes, we assume that, for each $(y,z)\in \{0,1\}\times \{0, 1\}$, $\mathcal{D}_{X|Y=y,Z=z}$ has the density function\footnote{Having a density function is closely related to absolute continuity.
In this work, we consider a probability distribution over $\mathbb{R}^n$ whose probability space is a triple $(\mathbb{R}^n, \mathcal{L}(\mathbb{R}^n), \nu)$, where $\mathcal{L}(\mathbb{R}^n)$ is the collection of all Lebesgue measurable sets, and the measure $\nu$ assigns the probability for $E\in \mathcal{L}(\mathbb{R}^n)$.
Then, by the Radon–Nikodym theorem~\cite{Billingsley1995PM}, the measure $\nu$ has the density function with respect to the Lebesgue measure $\mu$ if and only if $\nu$ is absolutely continuous with respect to $\mu$.} $f_{X|Y=y, Z=z}(x|y, z)$ with respect to the Lebesgue measure $\mu$ satisfying $\Pr(X\in E|Y=y, Z=z)=\int_E f_{X|Y=y, Z=z} \diff\mu$ for any Lebesgue measurable set $E\in\mathbb{R}^n$.
Then the joint density function $f(x,y,z)$ of $\mathcal{D}$ is $f_{X|Y=y, Z=z}(x|y, z) \Pr_{\mathcal{D}}(Y=y, Z=z)$, and the marginal density function of $X$, denoted $f_X(x)$, is $\sum_{(y,z)\in \{0,1\}\times \{0, 1\}} f(x,y,z)$.

\subsection{Learning a model with fairness constraints}
In this work, we consider a model $h:\mathcal{X}\rightarrow \mathcal{Y}$ that does not suffer from disparate treatment, \emph{i.e.,} $h$ does not take the sensitive attribute $z \in \mathcal{Z}$ as input.
Let $\mathcal{H}$ be the hypothesis class.
Let $\ell: \mathcal{Y}\times \mathcal{Y} \rightarrow \{0,1\}$ be the $0/1$ loss function, \emph{i.e.,} $\ell(\hat{y}, y)=\mathbb{1}(y\neq \hat{y})$ where $\mathbb{1}(\cdot)$ is the indicator function.
Let $R_{\ell}(h;\mathcal{D})$ be the true risk of $h$ on $\mathcal{D}$, \emph{i.e.,} $R_{\ell}(h;\mathcal{D})=\mathbb{E}_{\mathcal{D}}[\ell(h (X), Y)]$.
We build our theory upon \emph{equal opportunity}~\cite{Hardt_2016}, but our analysis can be generalized to \emph{demographic parity}~\cite{Feldman_2015} (see Appendix~\ref{sec:A} for details).
A model $h: \mathcal{X}\rightarrow \mathcal{Y}$ satisfies equal opportunity on the distribution $\mathcal{D}$ if
$$\Pr_{\mathcal{D}} (h(X)=1|Y=1, Z=0) = \Pr_{\mathcal{D}} (h(X)=1|Y=1, Z=1).$$
We measure the unfairness of a model by capturing the dissimilarity between true positive rates across the sensitive attributes, which is similar to methods used in~\cite{Chang2020Adversarial,Roh2020FairBatch,roh2021sample}.

\begin{definition}\label{def:1}
The \emph{fairness gap} of a model $h:\mathcal{X}\rightarrow \mathcal{Y}$ on the distribution $\mathcal{D}$, denoted $\Delta(h,\mathcal{D})$, is
$$\max_{z\in \mathcal{Z}} \Big|\Pr_{\mathcal{D}} (h(X)=1|Y=1, Z=z) - \Pr_{\mathcal{D}} (h(X)=1|Y=1) \Big|.$$
For $\delta \in [0, 1]$, $h$ is \emph{$\delta$-fair} on $\mathcal{D}$ if $\Delta(h, \mathcal{D}) \le \delta$.
The model $h$ is \emph{perfectly fair} on $\mathcal{D}$ if it is $0$-fair.
We similarly define the fairness gap, $\delta$-fairness, and perfect fairness of $h$ on the training set $D$ by using the empirical probability $\Pr_D(\cdot)$ over $D$.
\end{definition}

\textbf{Learner:}
We assume that the learner can solve any optimization problem with infinite computing power.
Moreover, the learner's hypothesis class $\mathcal{H}$ consists of some Lebesgue measurable functions from $\mathcal{X}$ to $\mathcal{Y}$, so any $h\in \mathcal{H}$ is deterministic.
The learner's goal is to find the model in $\mathcal{H}$ that achieves the minimum \emph{true} risk among perfectly fair models, which we call Fair True Risk Minimization (FTRM), by solving the following constrained optimization problem:
\begin{align}\label{eq:1}
  \min_h \{R_\ell(h; \mathcal{D}) \colon\ h\in \mathcal{H}, h \text{~is perfectly fair on~}\mathcal{D}\}.
\end{align}
We denote the set of solutions of \eqref{eq:1} by $\mathcal{A}_{0}(\mathcal{D})$.
Moreover, we define $\mathcal{A}_{\delta}(\mathcal{D})$ as the set of solutions of $\min_h\{R_{\ell}(h;\mathcal{D}) \colon\ h\in \mathcal{H}, h \text{~is $\delta$-fair on $\mathcal{D}$}\}$.
Note that $\mathcal{A}_{1}(\mathcal{D})$ is the set of unconstrained true risk minimizers since any model is $1$-fair.

\subsection{Flipping attacks}
We consider flipping attacks which belong to data poisoning attacks.
\begin{definition}
Let $\mathcal{D}$ and $\mathcal{D}'$ be probability distributions over $\mathcal{X}\times\mathcal{Y}\times\mathcal{Z}$ with density functions $f$ and $f'$, respectively.

(i) We say $\mathcal{D}'$ is obtained by \emph{flipping attack} on $\mathcal{D}$ if $f_X(x)$ and $f_X'(x)$, the marginal density functions of $X$, are the same almost everywhere in $\mathcal{X}$, \emph{i.e.,} $\mu(\mathcal{X}\setminus \{x\in \mathcal{X}\colon\ f_X(x)= f_X'(x)\})= 0$.

There are three pure flipping attacks as follows.

(ii) We say $\mathcal{D}'$ is obtained by \emph{pure $Y$-flipping attack} on $\mathcal{D}$ if $$\mu(\mathcal{X}\setminus \{x\in \mathcal{X}\colon\ f(x,0,0)+f(x,1,0)=f'(x,0,0)+f'(x,1,0), f(x,0,1)+f(x,1,1)=f'(x,0,1)+f'(x,1,1) \})= 0.$$

(iii) We say $\mathcal{D}'$ is obtained by \emph{pure $Z$-flipping attack} on $\mathcal{D}$ if $$\mu(\mathcal{X}\setminus \{x\in \mathcal{X}\colon\ f(x,0,0)+f(x,0,1)=f'(x,0,0)+f'(x,0,1), f(x,1,0)+f(x,1,1)=f'(x,1,0)+f'(x,1,1) \})= 0.$$

(iv) We say $\mathcal{D}'$ is obtained by \emph{pure $(Y\&Z)$-flipping attack} on $\mathcal{D}$ if $$\mu(\mathcal{X}\setminus \{x\in \mathcal{X}\colon\ f(x,0,0)+f(x,1,1)=f'(x,0,0)+f'(x,1,1), f(x,0,1)+f(x,1,0)=f'(x,0,1)+f'(x,1,0) \})= 0.$$
\end{definition}

Let us interpret pure flipping attacks.
For example, consider the case where $\mathcal{D}'$ is obtained by pure $Y$-flipping attack on $\mathcal{D}$.
By definition, $f(x,0,0)+f(x,1,0)=f'(x,0,0)+f'(x,1,0)$ and $f(x,0,1)+f(x,1,1)=f'(x,0,1)+f'(x,1,1)$ almost everywhere in $\mathcal{X}$.
Then there exist $a_x \in [-f(x,1,0), f(x,0,0)]$ and $b_x \in [-f(x,1,1), f(x,0,1)]$ such that $f'(x,0,0)=f(x,0,0)-a_x$, $f'(x,1,0)=f(x,1,0)+a_x$, $f'(x,0,1)=f(x,0,1)-b_x$, and $f'(x,1,1)=f(x,1,1)+b_x$.
For simplicity, we assume that $a_x \in [0, f(x,0,0)]$ and $b_x \in [0, f(x,0,1)]$.
Then $\mathcal{D}'$ can be realized by flipping the $Y$ value with probability $\sfrac{a_x}{f(x,0,0)}$ when $X=x, Y=0, Z=0$ and flipping the $Y$ value with probability $\sfrac{b_x}{f(x,0,1)}$ when $X=x, Y=0, Z=1$.
Other pure flipping attacks can be interpreted similarly.
Below we provide a toy example showing the effect of flipping attack schemes defined above.

\begin{example}\label{ex:1}
Let $\mathcal{D}$ be a probability distribution over $\mathcal{X}\times\mathcal{Y}\times\mathcal{Z}$ where samples with $Y=y$ and $Z=z$ are uniformly distributed with density of $1$ on the square region in the $k$-th quadrant, where $k = 3 - y + z - 2yz$.
Shown in Fig.~\ref{fig:2a} is a visualization of $\mathcal{D}$.
Fig.~\ref{fig:2b} shows the distribution obtained by pure $Y$-flipping attack on $\mathcal{D}$ ($Y$ values in circular regions in the second and fourth quadrants are flipped).
Similarly, Fig.~\ref{fig:2c} shows the distribution obtained by pure $Z$-flipping attack on $\mathcal{D}$, and Fig.~\ref{fig:2d} shows the distribution obtained by pure $(Y\&Z)$-flipping attack on $\mathcal{D}$.
Lastly, Fig.~\ref{fig:2e} shows the distribution obtained by flipping attack which is not pure.
\end{example}

\begin{figure}[t]
    \centering
    \begin{subfigure}[b]{0.19\columnwidth}
        \includegraphics[width=\columnwidth]{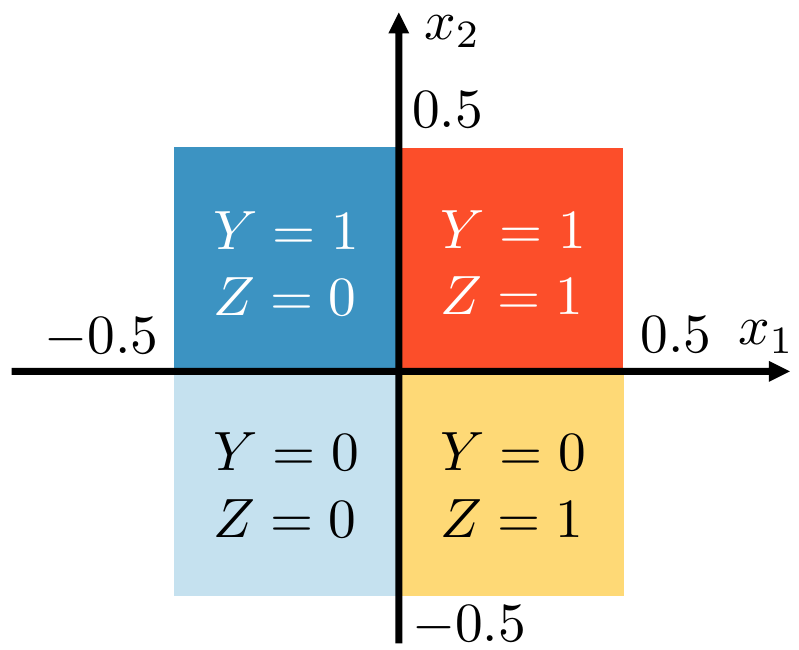}
        \caption{Uncorrupted data $\mathcal{D}$}
        \label{fig:2a}
    \end{subfigure}
    \begin{subfigure}[b]{0.19\columnwidth}
        \includegraphics[width=\columnwidth]{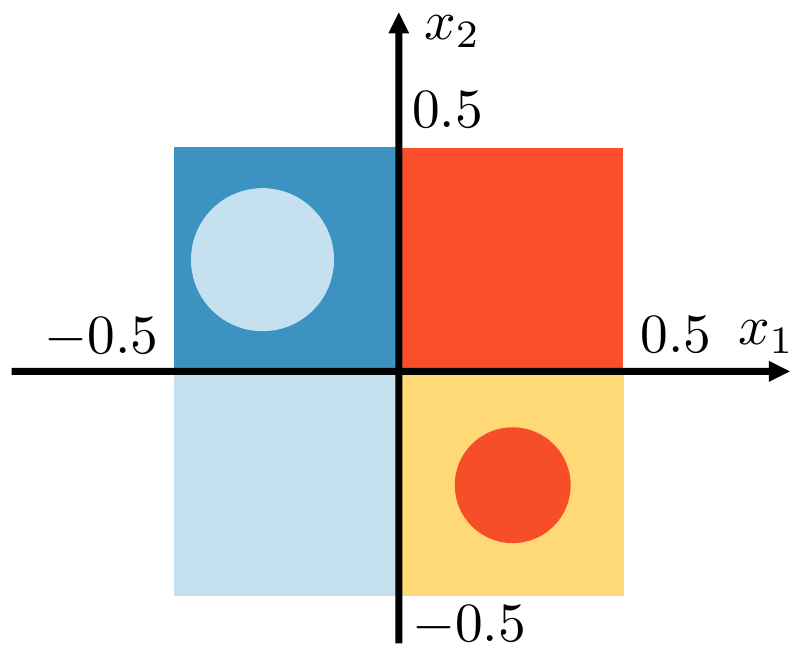}
        \caption{Pure $Y$-flipping}
        \label{fig:2b}
    \end{subfigure}
    \begin{subfigure}[b]{0.19\columnwidth}
        \includegraphics[width=\columnwidth]{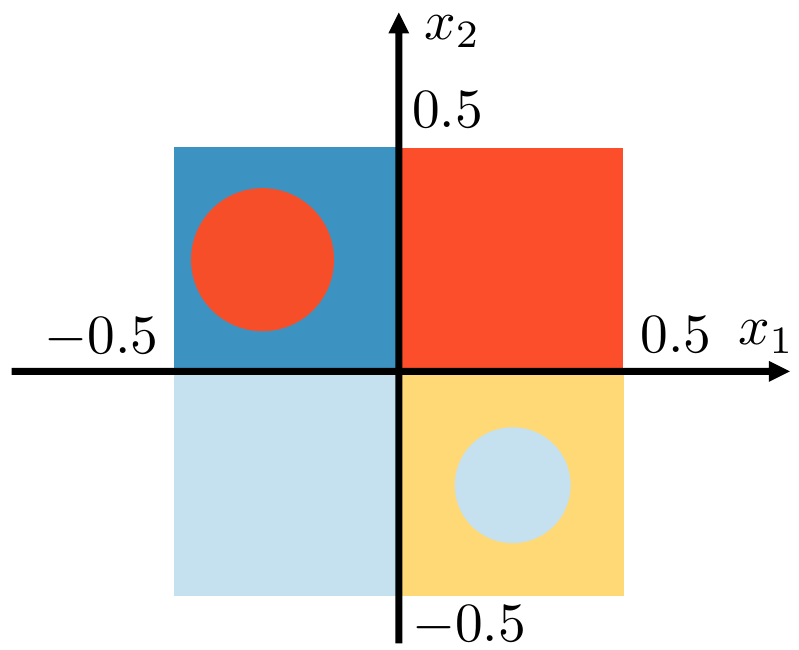}
        \caption{Pure $Z$-flipping}
        \label{fig:2c}
    \end{subfigure}
    \begin{subfigure}[b]{0.19\columnwidth}
        \includegraphics[width=\columnwidth]{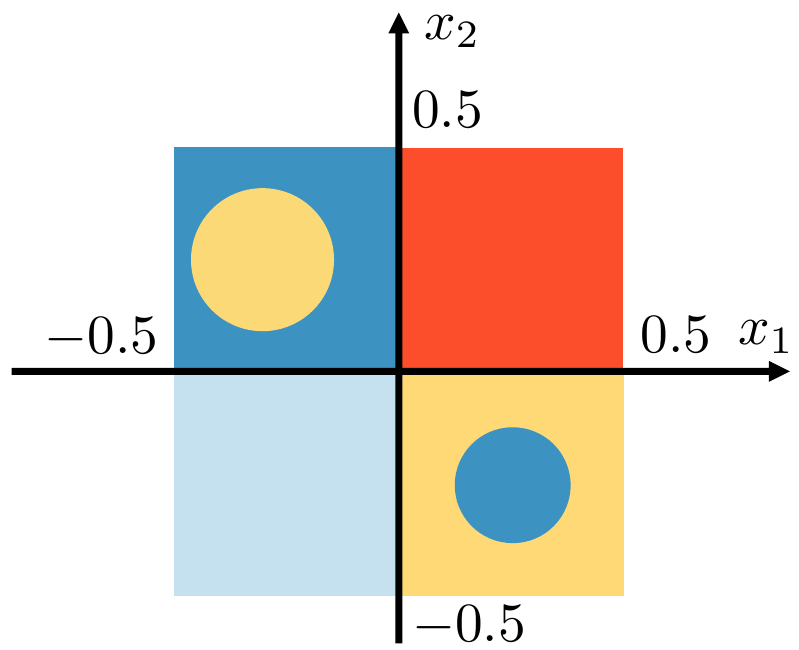}
        \caption{Pure $(Y\&Z)$-flipping}
        \label{fig:2d}
    \end{subfigure}
    \begin{subfigure}[b]{0.19\columnwidth}
        \includegraphics[width=\columnwidth]{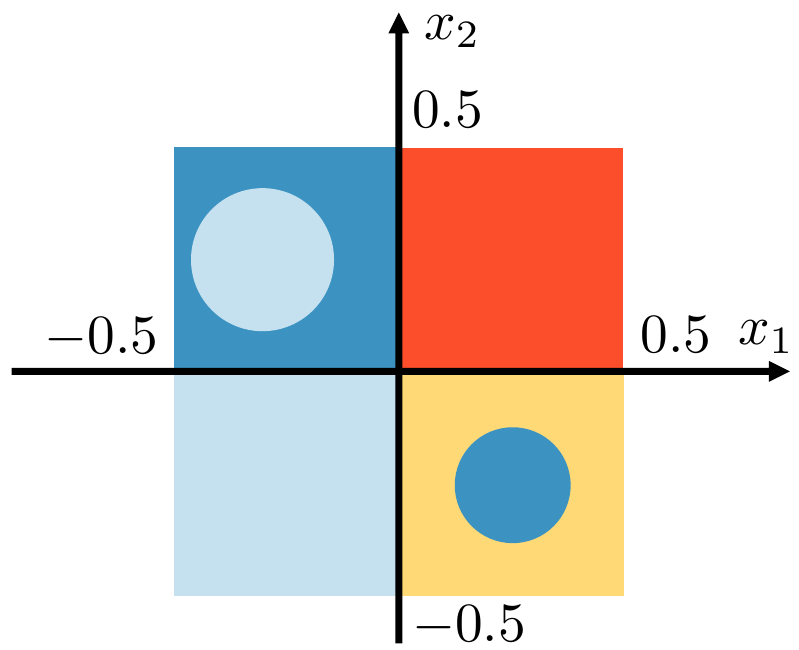}
        \caption{General flipping}
        \label{fig:2e}
    \end{subfigure}    
    \caption{Flipping attacks on $\mathcal{D}$ when $\mathcal{D}$ follows the description in Example~\ref{ex:1}.
    }
    \label{fig:2}
\end{figure}

\textbf{Attacker:}
The attacker knows the entire learning procedure (white-box attack) and can make the learner train the model on another distribution $\mathcal{D}'$ with the following constraints.
(1) The conditional distribution $\mathcal{D}'_{X|Y=y,Z=z}$ has a density function with respect to the Lebesgue measure $\mu$ for each $(y,z) \in \mathcal{Y}\times \mathcal{Z}$; if this does not hold, the attack may be easily detected by the learner.
(2) The distribution $\mathcal{D}'$ is obtained by flipping attacks on $\mathcal{D}$, \emph{i.e.,} $\mathcal{D}'_X=\mathcal{D}_X$.
Thus, the attacker's search space $\mathcal{S}$ is 
\begin{align}
    \mathcal{S}&=\{\mathcal{D}'\colon\  \mathcal{D}' \text{~is a prob. dist. over~} \mathcal{X}\times\mathcal{Y}\times\mathcal{Z}, \mathcal{D}'_X=\mathcal{D}_X, \mathcal{D}'_{X|Y=y,Z=z} \text{~has a density w.r.t. $\mu$~} \forall (y,z) \in \mathcal{Y}\times \mathcal{Z}\} \label{eq:2}
\end{align}
The attacker's goal is to make the learner output the target model $h_{\text{target}}$ with the minimum amount of data perturbation, measured in the total variation distance.
For two distributions $\mathcal{D}_1$ and $\mathcal{D}_2$ over $\mathcal{X}\times\mathcal{Y}\times\mathcal{Z}$, the total variation distance between $\mathcal{D}_1$ and $\mathcal{D}_2$, denoted $d_{\text{TV}}(\mathcal{D}_1, \mathcal{D}_2)$, is
\begin{align}
\frac{1}{2} \sum_{(y,z)\in \mathcal{Y}\times\mathcal{Z}} \int_{\mathbb{R}^n}\left|f_1(x,y,z)-f_2(x,y,z)\right| \diff\mu,
\end{align} 
where $f_1(x,y,z)$ and $f_2(x,y,z)$ are (mixed) joint density functions of $\mathcal{D}_1$ and $\mathcal{D}_2$, respectively.
Hence the attacker solves the following bilevel optimization problem:
\begin{align}\label{eq:4}
\min_{\mathcal{D}'} \big\{d_{\text{TV}}(\mathcal{D}, \mathcal{D}') \colon\ \mathcal{D}'\in \mathcal{S}, \mathcal{A}_{0}(\mathcal{D}')= \{h_{\text{target}}\} \big\}.
\end{align}
Define the infimum of the objective function of \eqref{eq:4} as  \begin{align}
    d^\star_{\text{TV}}(h_{\text{target}}) = \inf_{\mathcal{D}'\in \Lambda_{0}(h_{\text{target}})} d_{\text{TV}}(\mathcal{D}, \mathcal{D}'),    
\end{align}
where $\Lambda_{\delta}(h):=\{\mathcal{D}' \colon\ \mathcal{D}'\in \mathcal{S}, \mathcal{A}_{\delta}(\mathcal{D}')= \{h\}\}$.
In other words, $d^\star_{\text{TV}}(h_{\text{target}})$ is the minimum amount of data perturbation for FTRM to output the target model $h_{\text{target}}\in \mathcal{H}$.

\section{Main Results}\label{sec:4}
In this section, we analyze $d^\star_{\text{TV}}(h)$ for a general target model $h\in\mathcal{H}$.
The following theorem provides the lower and upper bounds on $d^\star_{\text{TV}}(h)$.

\begin{theorem}\label{thm:1}
Let $h\in \mathcal{H}$.
Then,
\begin{align*}
C(h, \mathcal{D}) \le d^\star_{\text{TV}}(h) \le \inf_{\widetilde{\mathcal{D}}\in \Lambda_{1}(h)} \left(d_{\text{TV}}(\mathcal{D}, \widetilde{\mathcal{D}}) + C(h, \widetilde{\mathcal{D}})\right)
\end{align*}
\begin{align*}
\text{where~} &C(h, \mathcal{D}) := \frac{|p_hs_h-q_hr_h|}{\max\{p_h+r_h, q_h+s_h\}},\\
 &p_h={\textstyle\Pr_{\mathcal{D}}}(h(X)=0, Y=1, Z=0),q_h= {\textstyle\Pr_{\mathcal{D}}}(h(X)=1, Y=1, Z=0),\\
&r_h={\textstyle\Pr_{\mathcal{D}}}(h(X)=0, Y=1, Z=1),s_h= {\textstyle\Pr_{\mathcal{D}}}(h(X)=1, Y=1, Z=1),\\
&\text{and~} C(h, \widetilde{\mathcal{D}}) \text{~is defined similarly over~} \widetilde{\mathcal{D}}.
\end{align*}
\end{theorem}

\begin{proof}
The lower bound is derived in Sec.~\ref{sec:4:a} by finding the minimum amount of data perturbation required for making $h$ look perfectly fair on the poisoned distribution.
The upper bound is derived in Sec.~\ref{sec:4:b} by constructing an explicit data distribution via two-stage attack algorithm.
\end{proof}

We note that our bounds on $d^\star_{\text{TV}}(h)$ in Thm.~\ref{thm:1} can possibly be loose.
For example, let $\mathcal{X}=\mathbb{R}^n$ for $n\ge 2$, and $\mathcal{H}$ be the set of all Lebesgue measurable functions from $\mathcal{X}$ to $\mathcal{Y}$.
Consider a distribution $\mathcal{D}$ on which its unconstrained risk minimizer is not perfectly fair, and let $h$ be the model that achieves the minimum risk among perfectly fair models.
Then, it is clear that $d^\star_{\text{TV}}(h)=0$; the attacker does not need to poison $\mathcal{D}$ at all.
Since $h$ is not equal to the unconstrained risk minimizer on $\mathcal{D}$, $d_{\text{TV}}(\mathcal{D}, \widetilde{\mathcal{D}})>0$ for any $\widetilde{\mathcal{D}}\in \Lambda_{1}(h)$.
Thus, we have ${\textstyle\inf_{\widetilde{\mathcal{D}}\in \Lambda_{1}(h)}} (d_{\text{TV}}(\mathcal{D}, \widetilde{\mathcal{D}}) + C(h, \widetilde{\mathcal{D}}))>0$, and $d^\star_{\text{TV}}(h) \lneq {\textstyle\inf_{\widetilde{\mathcal{D}}\in \Lambda_{1}(h)}} (d_{\text{TV}}(\mathcal{D}, \widetilde{\mathcal{D}}) + C(h, \widetilde{\mathcal{D}}))$.
This shows that the bounds in Thm.~\ref{thm:1} are not tight in general.
However, when $h$ is the unique unconstrained risk minimizer, our bounds are tight by the following corollary.

\begin{corollary}\label{cor:1}
Let $h^*$ be the unique unconstrained risk minimizer, \emph{i.e.,}  $\arg\min_{g\in \mathcal{H}} R_{\ell}(g; \mathcal{D}) = \{h^*\}$. Then, $d^\star_{\text{TV}}(h^*) = C(h^*,\mathcal{D})$.
\end{corollary}

\begin{proof}
Recall that $\Lambda_{1}(h)=\{\mathcal{D}' \colon\ \mathcal{D}'\in \mathcal{S}, \mathcal{A}_{1}(\mathcal{D}')= \{h\}\}$, and $\mathcal{A}_{1}(\mathcal{D})$ is the set of unconstrained true risk minimizers.
Since $h^*$ is the unique unconstrained risk minimizer on $\mathcal{D}$,
$\mathcal{D}\in \Lambda_{1}(h^*)$.
Thus our upper bound on $d^\star_{\text{TV}}(h^*)$ given by Thm.~\ref{thm:1} is upper bounded by
$$d_{\text{TV}}(\mathcal{D}, \mathcal{D}) + C(h^*, \mathcal{D})= C(h^*, \mathcal{D}),$$
and this is equal to the lower bound on $d^\star_{\text{TV}}(h^*)$ given by Thm.~\ref{thm:1}.
Therefore, $d^\star_{\text{TV}}(h^*) = C(h^*,\mathcal{D})$.
\end{proof}

\subsection{Lower bound on $d^\star_{\text{TV}}(h)$}\label{sec:4:a}
The following lemma provides the key inequality to derive the lower bound on $d^\star_{\text{TV}}(h)$.
\begin{lemma}\label{lem:1}
Let $\mathcal{X}= \mathbb{R}^n$, $\mathcal{Y}=\{0,1\}$, $\mathcal{Z}=\{0,1\}$, and $\mathcal{D}' \in \mathcal{S}$.
If $h \in \mathcal{H}$ is perfectly fair on $\mathcal{D}'$, then 
\begin{align}\label{eq:6}
    d_{\text{TV}}(\mathcal{D}, \mathcal{D}')\ge  C(h, \mathcal{D})
\end{align}
\vskip -0.2in
\end{lemma}
\begin{proof}
Let $f$ and $f'$ be density functions of $\mathcal{D}$ and $\mathcal{D}'$, respectively.
Let $A=\{x\in \mathbb{R}^n: h(x)=1\}$.
For $(a,b,c)\in \{0,1\}\times\{0,1\}\times\{0,1\}$, let $p_{a,b,c}=\Pr_{\mathcal{D}}(h(X)=a, Y=b, Z=c)$, $p'_{a,b,c}=\Pr_{\mathcal{D}'}(h(X)=a, Y=b, Z=c)$, and $\alpha_{a,b,c}= p'_{a,b,c} - p_{a,b,c}$.
Then $d_{\text{TV}}(\mathcal{D}, \mathcal{D}')$ can be lower bounded as follows.
\begingroup
\allowdisplaybreaks
\begin{align}
&d_{\text{TV}}(\mathcal{D}, \mathcal{D}')=\frac{1}{2} \sum_{(y,z)\in \{0,1\}\times\{0,1\}} \int_{\mathbb{R}^n}|f(x,y,z)-f'(x,y,z)|\diff\mu \nonumber\\
&= \frac{1}{2} \sum_{(y,z)\in \{0,1\}\times\{0,1\}} \int_{A}|f(x,y,z)-f'(x,y,z)|\diff\mu + \frac{1}{2} \sum_{(y,z)\in \{0,1\}\times\{0,1\}} \int_{\mathbb{R}^n \setminus A}|f(x,y,z)-f'(x,y,z)|\diff\mu \nonumber\\
&\overset{(i)}{\ge} \frac{1}{2} \sum_{(y,z)\in \{0,1\}\times\{0,1\}} \Big|\int_{A}\big(f(x,y,z)-f'(x,y,z)\big) \diff\mu\Big| + \frac{1}{2} \sum_{(y,z)\in \{0,1\}\times\{0,1\}} \Big|\int_{\mathbb{R}^n \setminus A}\big(f(x,y,z)-f'(x,y,z)\big) \diff\mu\Big| \nonumber\\
&= \frac{1}{2} \sum_{(a,b,c)\in \{0,1\}\times\{0,1\}\times\{0,1\}} \Big|\Pr_{\mathcal{D}}(h(X)=a, Y=b, Z=c) -\Pr_{\mathcal{D}'}(h(X)=a, Y=b, Z=c) \Big| \nonumber\\
&= \frac{1}{2} \sum_{(a,b,c)\in \{0,1\}\times\{0,1\}\times\{0,1\}} |\alpha_{a,b,c}|, \label{lowerbddTV}
\end{align}
\endgroup
where (i) comes from the triangle inequality.
Since $\mathcal{D}' \in \mathcal{S}$, we have $\mathcal{D}_X = \mathcal{D}_X'$.
This implies $\Pr_{\mathcal{D}} (h(X)=1)=\Pr_{\mathcal{D}'} (h(X)=1)$, which is equivalent to
\begin{equation}\label{fairconstraint1}
    \sum_{(b,c)\in \{0,1\}\times\{0,1\}} \alpha_{1,b,c} = \sum_{(b,c)\in \{0,1\}\times\{0,1\}} \alpha_{0,b,c} = 0.
\end{equation}
Moreover, $h$ is perfectly fair on $\mathcal{D}'$ if and only if
$$\frac{p'_{1,1,0}}{p'_{0,1,0}+p'_{1,1,0}}=\frac{p'_{1,1,1}}{p'_{0,1,1}+p'_{1,1,1}},$$
which is equivalent to
\begin{equation}\label{fairconstraint3}
p'_{1,1,0}\cdot p'_{0,1,1} = p'_{0,1,0}\cdot p'_{1,1,1}.
\end{equation}
It suffices to show that \eqref{lowerbddTV} is lower bounded by $C(h, \mathcal{D})$ under the constraints \eqref{fairconstraint1} and \eqref{fairconstraint3}.
It is not easy to analyze \eqref{lowerbddTV} directly because it is the sum of 8 unknown variables.
Hence we define $\beta_{a,b,c}$ as
$$\begin{cases}
\alpha_{a,b,c}+\frac{p'_{0,1,0}}{p'_{0,1,0}+p'_{0,1,1}}(\alpha_{1,0,0}+\alpha_{1,0,1}) \text{~~if~~} (a,b,c)= (1,1,0)\\
\alpha_{a,b,c}+\frac{p'_{0,1,1}}{p'_{0,1,0}+p'_{0,1,1}}(\alpha_{1,0,0}+\alpha_{1,0,1}) \text{~~if~~} (a,b,c)= (1,1,1)\\
\alpha_{a,b,c}+\frac{p'_{0,1,0}}{p'_{0,1,0}+p'_{0,1,1}}(\alpha_{0,0,0}+\alpha_{0,0,1}) \text{~~if~~} (a,b,c)= (0,1,0)\\
\alpha_{a,b,c}+\frac{p'_{0,1,1}}{p'_{0,1,0}+p'_{0,1,1}}(\alpha_{0,0,0}+\alpha_{0,0,1}) \text{~~if~~} (a,b,c)= (0,1,1)\\
0 \text{~~if~~} b=0, (a,c)\in \{0,1\}\times\{0,1\} \end{cases}$$
By definition of $\beta_{a,b,c}$ and \eqref{fairconstraint1}, it is clear that $\beta_{1,1,0}=-\beta_{1,1,1}$ and $\beta_{0,1,0}=-\beta_{0,1,1}$.
Then \eqref{lowerbddTV} is lower bounded as follows.
\begingroup
\allowdisplaybreaks
\begin{align}
    &\frac{1}{2} \sum_{(a,b,c)\in \{0,1\}\times\{0,1\}\times\{0,1\}} |\alpha_{a,b,c}|  \nonumber\\
    &= \frac{1}{2}\Big( |\alpha_{1,1,0}|+\frac{p'_{0,1,0}}{p'_{0,1,0}+p'_{0,1,1}}(|\alpha_{1,0,0}|+|\alpha_{1,0,1}|) \Big) + \frac{1}{2}\Big(|\alpha_{1,1,1}|+\frac{p'_{0,1,1}}{p'_{0,1,0}+p'_{0,1,1}}(|\alpha_{1,0,0}|+|\alpha_{1,0,1}|) \Big)\nonumber\\
    &~~~~ + \frac{1}{2}\Big( |\alpha_{0,1,0}|+\frac{p'_{0,1,0}}{p'_{0,1,0}+p'_{0,1,1}}(|\alpha_{0,0,0}|+|\alpha_{0,0,1}|) \Big) + \frac{1}{2} \Big(|\alpha_{0,1,1}|+\frac{p'_{0,1,1}}{p'_{0,1,0}+p'_{0,1,1}}(|\alpha_{0,0,0}|+|\alpha_{0,0,1}|) \Big)\nonumber\\
    &\overset{(ii)}{\ge} \frac{1}{2}\Big|\alpha_{1,1,0}+\frac{p'_{0,1,0}}{p'_{0,1,0}+p'_{0,1,1}}(\alpha_{1,0,0}+\alpha_{1,0,1})\Big| + \frac{1}{2}\Big|\alpha_{1,1,1}+\frac{p'_{0,1,1}}{p'_{0,1,0}+p'_{0,1,1}}(\alpha_{1,0,0}+\alpha_{1,0,1})\Big| \nonumber\\
    &~~~~+ \frac{1}{2}\Big| \alpha_{0,1,0}+\frac{p'_{0,1,0}}{p'_{0,1,0}+p'_{0,1,1}}(\alpha_{0,0,0}+\alpha_{0,0,1}) \Big| + \frac{1}{2}\Big|\alpha_{0,1,1}+\frac{p'_{0,1,1}}{p'_{0,1,0}+p'_{0,1,1}}(\alpha_{0,0,0}+\alpha_{0,0,1}) \Big| \nonumber\\
    &= \frac{1}{2} (|\beta_{1,1,0}|+|\beta_{1,1,1}|+|\beta_{0,1,0}|+|\beta_{0,1,1}|)\nonumber\\
    &= |\beta_{1,1,1}|+|\beta_{0,1,1}|,\label{betasimplebddTV}
\end{align}
\endgroup
where (ii) comes from the triangle inequality, and the last equality comes from $\beta_{1,1,0}=-\beta_{1,1,1}$ and $\beta_{0,1,0}=-\beta_{0,1,1}$.
We now show how the constraints \eqref{fairconstraint1} and \eqref{fairconstraint3} are used to find the constraint on $\beta_{1,1,1}$ and $\beta_{0,1,1}$.

A direct calculation yields
\begingroup
\allowdisplaybreaks
\begin{align*}
    &p_{1,1,0}+\beta_{1,1,0} = p_{1,1,0}+\alpha_{1,1,0}+\frac{p'_{0,1,0}(\alpha_{1,0,0}+\alpha_{1,0,1})}{p'_{0,1,0}+p'_{0,1,1}}\\
    &=p'_{1,1,0}+\frac{p'_{0,1,0}(\alpha_{1,0,0}+\alpha_{1,0,1})}{p'_{0,1,0}+p'_{0,1,1}}\\
    &=\frac{\big(p'_{0,1,0}(p'_{1,1,0}+\alpha_{1,0,0}+\alpha_{1,0,1})+p'_{1,1,0}p'_{0,1,1}\big)}{p'_{0,1,0}+p'_{0,1,1}}\\
    &\overset{\eqref{fairconstraint3}}{=} \frac{\big(p'_{0,1,0}(p'_{1,1,0}+\alpha_{1,0,0}+\alpha_{1,0,1})+p'_{0,1,0}p'_{1,1,1}\big)}{p'_{0,1,0}+p'_{0,1,1}}\\
    &=\frac{p'_{0,1,0}(p'_{1,1,0}+p'_{1,1,1}+\alpha_{1,0,0}+\alpha_{1,0,1})}{p'_{0,1,0}+p'_{0,1,1}}\\
    &=\frac{p'_{0,1,0}(p_{1,1,0}+\alpha_{1,1,0}+p_{1,1,1}+\alpha_{1,1,1}+\alpha_{1,0,0}+\alpha_{1,0,1})}{p'_{0,1,0}+p'_{0,1,1}}\\
    &\overset{\eqref{fairconstraint1}}{=}\frac{p'_{0,1,0}(p_{1,1,0}+p_{1,1,1})}{p'_{0,1,0}+p'_{0,1,1}} 
\end{align*}
\endgroup
and 
\begingroup
\allowdisplaybreaks
\begin{align*}
    &p_{0,1,0}+\beta_{0,1,0} =p_{0,1,0}+\alpha_{0,1,0}+\frac{p'_{0,1,0}(\alpha_{0,0,0}+\alpha_{0,0,1})}{p'_{0,1,0}+p'_{0,1,1}}\\
    &= p'_{0,1,0}+\frac{p'_{0,1,0}(\alpha_{0,0,0}+\alpha_{0,0,1})}{p'_{0,1,0}+p'_{0,1,1}}\\
    &= \frac{p'_{0,1,0}(p'_{0,1,0}+p'_{0,1,1}+\alpha_{0,0,0}+\alpha_{0,0,1} )}{p'_{0,1,0}+p'_{0,1,1}} \\
    &= \frac{p'_{0,1,0}(p_{0,1,0}+\alpha_{0,1,0}+p_{0,1,1}+\alpha_{0,1,1}+\alpha_{0,0,0}+\alpha_{0,0,1} )}{p'_{0,1,0}+p'_{0,1,1}} \\
    &\overset{\eqref{fairconstraint1}}{=} \frac{p'_{0,1,0}(p_{0,1,0}+p_{0,1,1} )}{p'_{0,1,0}+p'_{0,1,1}}.
\end{align*}
\endgroup
Similarly, a direct calculation yields $$p_{1,1,1}+\beta_{1,1,1} = \frac{p'_{0,1,1}(p_{1,1,0}+p_{1,1,1})}{p'_{0,1,0}+p'_{0,1,1}}$$ and $$p_{0,1,1}+\beta_{0,1,1} =\frac{p'_{0,1,1}(p_{0,1,0}+p_{0,1,1})}{p'_{0,1,0}+p'_{0,1,1}}.$$
Combining the results above, one can obtain
$$\frac{(q+\beta_{1,1,0})}{(p+\beta_{0,1,0})+(q+\beta_{1,1,0})}=\frac{(s+\beta_{1,1,1})}{(r+\beta_{0,1,1})+(s+\beta_{1,1,1})},$$
where $p = p_{0,1,0}, q = p_{1,1,0}, r = p_{0,1,1}, s = p_{1,1,1}$.
Using the fact that $\beta_{1,1,0}=-\beta_{1,1,1}$ and $\beta_{0,1,0}=-\beta_{0,1,1}$, the constraint above can be written as
\begin{equation}\label{simplebetaconstraint2}
    \frac{(q-\beta_{1,1,1})}{(p-\beta_{0,1,1})+(q-\beta_{1,1,1
    })}=\frac{(s+\beta_{1,1,1})}{(r+\beta_{0,1,1})+(s+\beta_{1,1,1})}.
\end{equation}

Solving \eqref{simplebetaconstraint2} for $\beta_{0,1,1}$, we get

$$\beta_{0,1,1} = \frac{p+r}{q+s}\beta_{1,1,1} + \frac{ps-qr}{q+s}.$$
Then \eqref{betasimplebddTV} is equal to 
\begin{align}
    |\beta_{1,1,1}|+\left|\frac{p+r}{q+s}\beta_{1,1,1} + \frac{ps-qr}{q+s}\right|, \label{betasimplebddTV_new}
\end{align}
which is a piecewise linear function of $\beta_{1,1,1}$.
Hence the lower bound on \eqref{betasimplebddTV_new} can be easily found as follows.
\begin{itemize}
\item Case 1. $ps-qr \ge 0, \frac{p+r}{q+s} \ge 1$: \eqref{betasimplebddTV_new} achieves the minimum of $\frac{ps-qr}{p+r}$ at $\beta_{1,1,1} = \frac{qr-ps}{p+r}$ ($\beta_{0,1,1} = 0$).
\item Case 2. $ps-qr \ge 0, \frac{p+r}{q+s} < 1$: \eqref{betasimplebddTV_new} achieves the minimum of $\frac{ps-qr}{q+s}$ at $\beta_{1,1,1} = 0$ ($\beta_{0,1,1} = \frac{ps-qr}{q+s}$).
\item Case 3. $ps-qr < 0, \frac{p+r}{q+s} \ge 1$: \eqref{betasimplebddTV_new} achieves the minimum of $\frac{qr-ps}{p+r}$ at $\beta_{1,1,1} = \frac{qr-ps}{p+r}$ ($\beta_{0,1,1} = 0$).
\item Case 4. $ps-qr < 0, \frac{p+r}{q+s} < 1$: \eqref{betasimplebddTV_new} achieves the minimum of $\frac{qr-ps}{q+s}$ at $\beta_{1,1,1} = 0$ ($\beta_{0,1,1} = \frac{ps-qr}{q+s}$).
\end{itemize}
From the above cases, one can conclude that \eqref{betasimplebddTV_new} is lower bounded by
\begin{align*}
     \frac{|ps-qr|}{\max\{p+r, q+s\}} = C(h, \mathcal{D}).
\end{align*}
\end{proof}
For any $\mathcal{D}' \in \Lambda_{0}(h)$, $h$ is perfectly fair on $\mathcal{D}'$, so we have $d_{\text{TV}}(\mathcal{D}, \mathcal{D}') \ge C(h, \mathcal{D})$ by Lem.~\ref{lem:1}.
Then, the lower bound on $d^\star_{\text{TV}}(h)$ in Thm.~\ref{thm:1} can be obtained as follows:
\begin{align}\label{eq:8}
    d^\star_{\text{TV}}(h) =  \inf_{\mathcal{D}'\in \Lambda_{0}(h)}~ d_{\text{TV}}(\mathcal{D}, \mathcal{D}') \ge C(h, \mathcal{D}).
\end{align}
We now interpret our lower bound.
Observe that $C(h, \mathcal{D})$ equals
$$\frac{\Delta(h, \mathcal{D})\cdot \min\{\Pr_{\mathcal{D}}(Y=1, Z=0), \Pr_{\mathcal{D}}(Y=1, Z=1)\}}{\max\{\text{TPR}_{h,\mathcal{D}}, \text{FNR}_{h,\mathcal{D}} \}},$$ 
where $\text{TPR}_{h,\mathcal{D}}$ denotes the true positive rate $\Pr_{\mathcal{D}}(h(X)=1 | Y=1)$, and $\text{FNR}_{h,\mathcal{D}}$ denotes the false negative rate $\Pr_{\mathcal{D}}(h(X)=0 | Y=1)$.
Then $C(h, \mathcal{D})$ is proportional to the unfairness gap $\Delta(h, \mathcal{D})$. 
Intuitively, this makes sense as one must apply heavy distortion to the data distribution to make $h$ look perfectly fair if $h$ was highly unfair on the original data.
The other term $\min\{\Pr_{\mathcal{D}}(Y=1, Z=0), \Pr_{\mathcal{D}}(Y=1, Z=1)\}$ in the numerator captures how trivial the data poisoning task is.
For instance, if either of the two terms is close to $0$, then it becomes much easier to satisfy equal opportunity by making very little perturbation.

We now show how to construct the distribution $\mathcal{D}'$ that matches the lower bound in \eqref{eq:6}.
Given the distribution $\mathcal{D}$ with the the density function $f$, and the target model $h$, we construct the distribution $\text{Fair}_h(\mathcal{D})$ with the density function $f_h(x,y,z)$ defined as follows.

\noindent \textbf{Case 1.} $p_h+r_h\ge q_h+s_h$, $\frac{q_h}{p_h}\ge \frac{s_h}{r_h}$: Define

$$f_h(x,y,z):= f(x,y,z) + \mathbb{1}(h(x)=1, y=1)\cdot(2z-1)\cdot\frac{q_hr_h-p_hs_h}{(p_h+r_h)q_h} f(x,1,0).$$
\textbf{Case 2.} $p_h+r_h\ge q_h+s_h$, $\frac{q_h}{p_h}<\frac{s_h}{r_h}$: Define 

$$f_h(x,y,z):=f(x,y,z) + \mathbb{1}(h(x)=1, y=1)\cdot(1-2z)\cdot \frac{p_hs_h-q_hr_h}{(p_h+r_h)s_h} f(x,1,1).$$
\textbf{Case 3.} $p_h+r_h< q_h+s_h$, $\frac{q_h}{p_h}\ge \frac{s_h}{r_h}$: Define

$$f_h(x,y,z):=f(x,y,z) + \mathbb{1}(h(x) \neq 1, y=1)\cdot(1-2z)\cdot \frac{q_hr_h-p_hs_h}{(q_h+s_h)r_h} f(x,1,1).$$
\textbf{Case 4.} $p_h+r_h< q_h+s_h$, $\frac{q_h}{p_h}< \frac{s_h}{r_h}$: Define
$$f_h(x,y,z):=f(x,y,z) + \mathbb{1}(h(x) \neq 1, y=1)\cdot(2z-1)\cdot \frac{p_hs_h-q_hr_h}{(q_h+s_h)p_h} f(x,1,0).$$

The following lemma shows that $\text{Fair}_h(\mathcal{D})$ satisfies desired properties.

\begin{lemma}\label{lem:2} 
The following properties hold: (i) $\text{Fair}_h(\mathcal{D}) \in \mathcal{S}$, (ii) $h$ is perfectly fair on $\text{Fair}_h(\mathcal{D})$, and (iii) $d_{\text{TV}}(\mathcal{D}, \text{Fair}_h(\mathcal{D})) = C(h, \mathcal{D})$.
\end{lemma}

\begin{proof}
We provide the proof for the case where $p_h+r_h\ge q_h+s_h$ and $\frac{q_h}{p_h}\ge \frac{s_h}{r_h}$, and other cases can be handled in a similar way.

(i) $\text{Fair}_h(\mathcal{D}) \in \mathcal{S}$: 
It suffices to show that $\text{Fair}_h(\mathcal{D})_X=\mathcal{D}_X$, $\text{Fair}_h(\mathcal{D})$ is a probability distribution over $\mathcal{X}\times\mathcal{Y}\times\mathcal{Z}$, and $\text{Fair}_h(\mathcal{D})_{X|Y=y,Z=z}$ has a density function with respect to the Lebesgue measure $\mu$ for all $(y,z) \in \mathcal{Y}\times \mathcal{Z}$.
Then $f_h(x,y,z)$ can be computed as follows:
\begin{align*}
f_h(x,y,z)=\begin{cases}
\frac{p_hq_h+p_hs_h}{(p_h+r_h)q_h} f(x,1,0) &\text{~if~} h(x)=1, y=1, z=0\\
f(x,1,1) + \frac{q_hr_h-p_hs_h}{(p_h+r_h)q_h} f(x,1,0) &\text{~if~} h(x)=1, y=1, z=1\\
f(x,y,z) &\text{~otherwise}
\end{cases}.
\end{align*}

Then, one can check that $${f_h}_X(x)=f_h(x,0,0)+f_h(x,0,1)+f_h(x,1,0)+f_h(x,1,1)=f(x,0,0)+f(x,0,1)+f(x,1,0)+f(x,1,1)=f_X(x),$$ which means $\text{Fair}_h(\mathcal{D})_X=\mathcal{D}_X$.
Moreover, one can observe that $f_h$ is a nonnegative function, and 
$$\sum_{(y,z)\in \{0,1\}\times\{0,1\}} \int_{\mathbb{R}^n}f_h(x,y,z)\diff\mu = \int_{\mathbb{R}^n}{f_h}_X(x)\diff\mu = \int_{\mathbb{R}^n}{f}_X(x)\diff\mu = 1.$$
Hence $\text{Fair}_h(\mathcal{D})$ is a probability distribution over $\mathcal{X}\times\mathcal{Y}\times\mathcal{Z}$.
Let $A=\{x\in \mathbb{R}^n\colon\ h(x)=1\}$.
Since $h$ is a Lebesgue measurable function, $A$ is a Lebesgue measurable set, so $\mathbb{1}(h(x)=1)$ is a Lebesgue measurable function.
For each $(y,z)\in \{0,1\}\times \{0,1\}$, $f_{X|Y=y, Z=z}(x|y, z)$ is a Lebesgue measurable function because it is the density function of $\mathcal{D}_{X|Y=y,Z=z}$.
Then, one can check that ${f_h}_{X|Y=y, Z=z}(x|y, z) = \frac{1}{\Pr_{\text{Fair}_h(\mathcal{D})}(Y=y, Z=z)} f_h(x,y,z)$ is a Lebesgue measurable function, using the fact that the sum and product of Lebesgue measurable functions are Lebesgue measurable.
Hence $\text{Fair}_h(\mathcal{D})_{X|Y=y,Z=z}$ has a density function with respect to the Lebesgue measure $\mu$ for all $(y,z) \in \mathcal{Y}\times \mathcal{Z}$.
Therefore, we can conclude that $\text{Fair}_h(\mathcal{D})\in \mathcal{S}$.

(ii) $h$ is perfectly fair on $\text{Fair}_h(\mathcal{D})$:
A direct calculation yields $$\Pr_{\text{Fair}_h(\mathcal{D})}(h(X)=0, Y=1, Z=0) = \int_{A^c} f_h(x,1,0) \diff{\mu} = \int_{A^c} f(x,1,0) \diff{\mu} =p_h$$ and $$\Pr_{\text{Fair}_h(\mathcal{D})}(h(X)=1, Y=1, Z=0) =  \int_A f_h(x,1,0) \diff{\mu} = \frac{p_hq_h+p_hs_h}{(p_h+r_h)q_h} \int_A f(x,1,0) \diff{\mu} = \frac{p_hq_h+p_hs_h}{p_h+r_h}.$$
Similarly, one can get $$\Pr_{\text{Fair}_h(\mathcal{D})}(h(X)=0, Y=1, Z=1) = r_h$$ and $$\Pr_{\text{Fair}_h(\mathcal{D})}(h(X)=1, Y=1, Z=1) = \frac{r_hq_h+r_hs_h}{p_h+r_h}.$$
Hence 
$$\Pr_{\text{Fair}_h(\mathcal{D})}(h(X)=1| Y=1, Z=0) = \Pr_{\text{Fair}_h(\mathcal{D})}(h(X)=1| Y=1, Z=1) = \frac{q_h+s_h}{p_h+q_h+r_h+s_h},$$
and this implies that $h$ is perfectly fair on $\text{Fair}_h(\mathcal{D})$.

(iii) $d_{\text{TV}}(\mathcal{D}, \text{Fair}_h(\mathcal{D})) = C(h, \mathcal{D})$:
Using the fact that $f_h(x,y,z)=f(x,y,z)$ for $h(x) \neq 1$ or $y\neq 1$, we get
\begin{align}
    d_{\text{TV}}(\mathcal{D}, \text{Fair}_h(\mathcal{D}))&= \frac{1}{2} \sum_{(y,z)\in \mathcal{Y}\times\mathcal{Z}} \int_{\mathbb{R}^n}\left|f(x,y,z)-f_h(x,y,z)\right| \diff\mu\nonumber\\
    &= \frac{1}{2} \sum_{z\in \{0,1\} } \int_{A}\left|f(x,1,z)-f_h(x,1,z)\right| \diff\mu. \label{eq:9}
\end{align}    
By a direct calculation, \eqref{eq:9} is equal to
\begingroup
\allowdisplaybreaks
\begin{align*}    
    &\frac{1}{2} \sum_{z\in \{0,1\} } \int_{A}\left|(2z-1)\frac{q_hr_h-p_hs_h}{(p_h+r_h)q_h} f(x,1,0)\right| \diff\mu\\
    &= \frac{1}{2} \sum_{z\in \{0,1\} } \int_{A}\frac{q_hr_h-p_hs_h}{(p_h+r_h)q_h} f(x,1,0) \diff\mu\\
    &= \frac{1}{2} \sum_{z\in \{0,1\} } \frac{q_hr_h-p_hs_h}{(p_h+r_h)q_h} \cdot q_h = \frac{q_hr_h-p_hs_h}{p_h+r_h} = C(h, \mathcal{D}).
\end{align*}
\endgroup
\end{proof}

Let us illustrate the density function $f_h$ of $\text{Fair}_h(\mathcal{D})$ for the case where $p_h+r_h\ge q_h+s_h$ and $\frac{q_h}{p_h}\ge \frac{s_h}{r_h}$.
If $h(x) \neq 1$ or $y\neq 1$, then the density function remains the same, \emph{i.e.,} $f_h = f$.
If $h(x)=1$ and $y=1$, then $f_h(x,1,0)=(1-\alpha) f(x,1,0)$ and $f_h(x,1,1)=f(x,1,1)+\alpha f(x,1,0)$, where $\alpha=\frac{q_hr_h-p_hs_h}{(p_h+r_h)q_h}$.
This can be interpreted as $\alpha$ fraction of the density at $(x,1,0)$ is transported to $(x,1,1)$.
In other words, this data distribution can be realized by flipping the $Z$ value with probability $\alpha$ when $X=x$, $Y=1$ and $h(x) = 1$.
This implies that pure $Z$-flipping attack is the optimal way of perturbing data distribution to make a target classifier look perfectly fair, and we will see a similar attack algorithm for the empirical risk case in Sec.~\ref{sec:5}.

\begin{remark}[Connection with Theorem~1 in~\cite{Wang2020Robust}]\label{rmk:1}
Theorem~1 in~\cite{Wang2020Robust} provides the lower bound on $d_{\text{TV}}(\mathcal{D}_{Z=z}, \mathcal{D}'_{Z=z})$ for each $z\in \mathcal{Z}$ when $h$ is perfectly fair on $\mathcal{D}'$.
However, they did not provide an explicit construction of $\mathcal{D}'$ that matches the bound.
Our construction scheme can be used to match their bound for certain cases, which we detail in Appendix~\ref{sec:B}.
\end{remark}

\subsection{Upper bound on $d^\star_{\text{TV}}(h)$}\label{sec:4:b}
By definition, $d^\star_{\text{TV}}(h)$ is upper bounded by $d_{\text{TV}}(\mathcal{D}, \mathcal{D}')$ for any $\mathcal{D}'\in \Lambda_{0}(h)$.
Hence we provide an upper bound on $d^\star_{\text{TV}}(h)$ by constructing a specific distribution $\mathcal{D}'$ that belongs to $\Lambda_{0}(h)$.
The distribution $\text{Fair}_h(\mathcal{D})$ defined in Sec.~\ref{sec:4:a} makes $h$ look perfectly fair with the minimum amount of data perturbation.
Assume a hypothetical scenario where $h$ is the \emph{only} perfectly fair model in the hypothesis class $\mathcal{H}$ on $\text{Fair}_h(\mathcal{D})$.
Then, $\mathcal{A}_{0}(\text{Fair}_h(\mathcal{D}))=\{h\}$ holds true.
So we get $\text{Fair}_h(\mathcal{D})\in \Lambda_{0}(h)$, and $d^\star_{\text{TV}}(h)$ could be upper bounded by $d_{\text{TV}}(\mathcal{D}, \text{Fair}_h(\mathcal{D}))$, which is equal to $C(h, \mathcal{D})$ by Lem.~\ref{lem:2}.
Unfortunately, this assumption does \emph{not} hold true by the following lemma; there are infinitely many perfectly fair classifiers (see Fig.~\ref{fig:3} for visualization).

\begin{figure}[t]
\centering
\includegraphics[width=0.5\columnwidth]{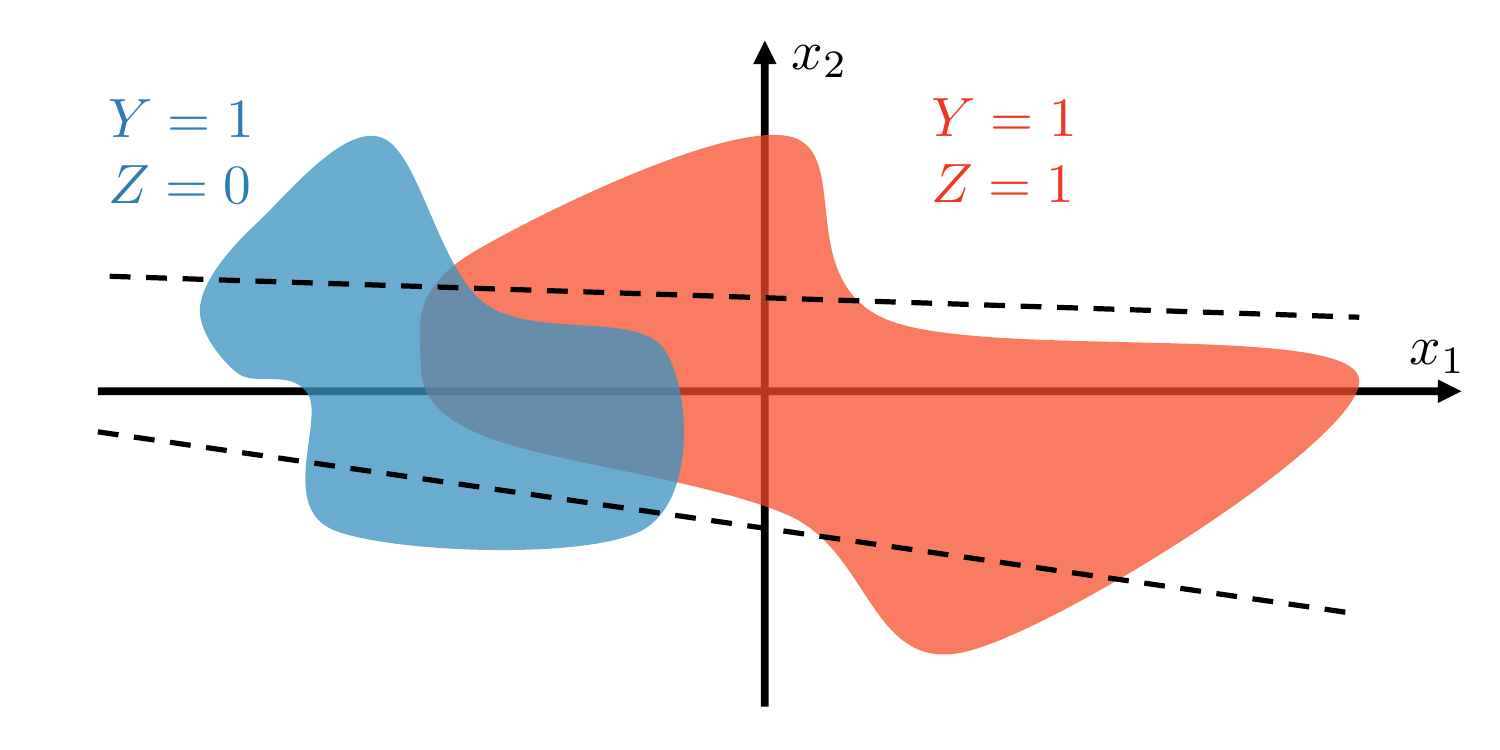}
\caption{Let $\mathcal{X}=\mathbb{R}^2, \mathcal{Y}=\mathcal{Z}=\{0,1\}$.
We consider a distribution $\mathcal{D}$ where samples with $Y=1,Z=0$ are uniformly distributed with density of 1 on the blue region, and samples with $Y=1,Z=1$ are uniformly distributed with density of 1 on the red region.
By Lem.~\ref{lem:3}, there exist infinitely many linear classifiers that are perfectly fair such as dotted lines.}\label{fig:3}
\end{figure}

\begin{lemma}\label{lem:3}
Let $\mathcal{X}= \mathbb{R}^n$, $\mathcal{Y}=\{0,1\}$, $\mathcal{Z}=\{0,1, \dots, d-1\}$. Let $\mathcal{D}$ be a probability distribution over $\mathcal{X}\times\mathcal{Y}\times\mathcal{Z}$ whose conditional distribution $\mathcal{D}_{X|Y=y,Z=z}$ has a density function with respect to the Lebesgue measure $\mu$ for all $(y,z)\in \mathcal{Y}\times \mathcal{Z}$.
If $n \geq d+\mathbb{1}(d\geq 3)$, then there exist infinitely many linear classifiers that are perfectly fair on $\mathcal{D}$.
Moreover, for all $x\in \mathcal{X}$, there exist at least one perfectly fair linear classifier whose decision boundary passes through $x$.
\end{lemma}

\begin{proof}
Case 1. $d=2$:
For a fixed $x_0 \in \mathbb{R}^n$ $(n\ge 2)$, we define the linear classifier $h_{\theta, x_0}$ parametrized by $\theta\in [0, 2\pi)$ as follows:
$$h_{\theta, x_0}(x)=\begin{cases} 1 \text{~~if~~} (x-x_0)^T\cdot(\cos\theta, \sin\theta,0,\dots,0)\ge 0\\
0 \text{~~o.w.}\end{cases}.$$
Note that the decision boundary of this linear classifier contains $x_0$.
Let $F(\theta)=\Pr(h_{\theta,x_0}(X)=1~|~Y=1, Z=0)$, $G(\theta)=\Pr(h_{\theta,x_0}(X)=1~|~Y=1, Z=1)$, and $H(\theta) = F(\theta) - G(\theta)$.
Then $F$ and $G$ are continuous on $[0, 2\pi)$ because $\mathcal{D}_{X|Y=y,Z=z}$ has a density function with respect to the Lebesgue measure $\mu$ for each $(y,z)\in \mathcal{Y}\times \mathcal{Z}$.
Hence $H=F-G$ is also continuous on $[0, 2\pi)$.
It is clear that $F(\theta)+F(\theta+\pi)=1$ and $G(\theta)+G(\theta+\pi)=1$.
Since $H(0)=F(0)-G(0)=(1-F(\pi))-(1-G(\pi))=-H(\pi)$, we have $H(0)H(\pi)\le 0$.
So $H(\theta_{x_0})=0$ for some $\theta_{x_0}\in [0,\pi]$ by the intermediate value theorem, which means $F(\theta_{x_0})=G(\theta_{x_0})$.
Then $h_{\theta_{x_0}, x_0}$ is perfectly fair on $\mathcal{D}$, and its decision boundary contains $x_0$.
We just showed that, for any $x\in \mathbb{R}^n$, there exist a perfectly fair linear classifier $h_{\theta_x, x}$ whose decision boundary contains $x$, which we call the property ($\star$).
However, we cannot conclude that there exist infinitely many perfectly fair linear classifiers yet, because $s\neq t~(\in \mathbb{R}^n)$ does not guarantee that $h_{\theta_s, s} \neq h_{\theta_t, t}$.
Indeed, $h_{\theta_{s}, s}$ is equal to $h_{\theta_{t}, t}$ if and only if
\begin{align}\label{equalcondition}
 \theta_{s}= \theta_{t} \text{~~and~~} (s-t)^T \cdot (\cos\theta_{s}, \sin\theta_{s}, 0,\dots, 0) = 0.
\end{align}

We now prove that there exist infinitely many perfectly fair linear classifiers by constructing $\{h_{\theta_{x_i}, x_i}\}_{i=1}^{\infty}$ inductively.
Let $x_1= (0,\dots, 0)\in \mathbb{R}^n$.
Using the property ($\star$), we can find $h_{\theta_{x_1}, x_1}$ that is perfectly fair on $\mathcal{D}$.
In the $k$-th step for $k\ge 2$, we can pick some $x_k\in \mathbb{R}^n\setminus \big( \bigcup_{1\le i \le k-1} \{x\in \mathbb{R}^n\colon\ (x-x_i)^T \cdot (\cos\theta_{x_i}, \sin\theta_{x_i}, 0,\dots, 0) = 0\}\big)$.
This is possible because the Lebesgue measure of $\big( \bigcup_{1\le i \le k-1} \{x\in \mathbb{R}^n\colon\ (x-x_i)^T \cdot (\cos\theta_{x_i}, \sin\theta_{x_i}, 0,\dots, 0) = 0\}\big)$, the countable union of hyperplanes, is zero.
Then we can find $h_{\theta_{x_k}, x_k}$ that is perfectly fair on $\mathcal{D}$ using the property ($\star$).
We now show that $h_{\theta_{x_k}, x_k}\notin \{h_{\theta_{x_1}, x_1}, \dots, h_{\theta_{x_{k-1}}, x_{k-1}}\}$. Toward a contradiction, suppose that $h_{\theta_{x_k}, x_k}=h_{\theta_{x_i}, x_i}$ for some $i\in \{1, \dots, k-1\}$.
Then $\theta_{x_k}= \theta_{x_i}$ and $(x_k-x_i)^T \cdot (\cos\theta_{x_i}, \sin\theta_{x_i}, 0,\dots, 0) = 0$ by \eqref{equalcondition}.
This contradicts the inductive assumption that $x_k\notin \big( \bigcup_{1\le i \le k-1} \{x\in \mathbb{R}^n\colon\ (x-x_i)^T \cdot (\cos\theta_{x_i}, \sin\theta_{x_i}, 0,\dots, 0) = 0\}\big)$.
Therefore, in the $k$-th step, we can find  $h_{\theta_{x_k}, x_k}$ that is different from  $h_{\theta_{x_1}, x_1}, \dots, h_{\theta_{x_{k-1}}, x_{k-1}}$ and perfectly fair on $\mathcal{D}$.
Iterating above steps, we can find the countable set $\{h_{\theta_{x_i}, x_i}\}_{i=1}^{\infty}$ whose elements are perfectly fair on $\mathcal{D}$.

Case 2. $d\ge 3$: Suppose $n\ge d+1$.
The argument using the intermediate value theorem cannot be extended to the case $d\ge 3$, because we need to equalize more than two functions.
The Borsuk-Ulam theorem~\cite{Hatcher2002AT}, provided below as Lem.~\ref{lem:4}, can be applied to this case.

\begin{lemma}[Borsuk-Ulam theorem~\cite{Hatcher2002AT}]\label{lem:4} Let $S^n=\{x\in \mathbb{R}^{n+1}:\|x\|=1\}$ for $n\ge 1$. If $g: S^n \rightarrow R^n$ is continuous, then there exists $x\in S^n$ such that $g(-x)=g(x)$.
\end{lemma}

Let $S^d=\{x\in \mathbb{R}^{d+1}:\|x\|=1\}$.
Define the natural embedding $\iota: S^d \rightarrow \mathbb{R}^n$ by $\iota\big((x_1, \dots, x_{d+1})\big)=(x_1, \dots, x_{d+1}, 0, \dots, 0)$.
For any $x_0\in \mathbb{R}^n$, define the linear classifier parametrized by $\omega \in S^d$ as follows:
$$h_{\omega, x_0}(x)=\begin{cases} 1 \text{~~if~~} (x-x_0)^T\cdot \iota (\omega) \ge 0\\
0 \text{~~o.w.}\end{cases}.$$
Let $F_i(\omega)=\Pr(h_{\omega, x_0}(X)=1~|~Y=1, Z=i)$ for $0\le i \le d-1$.
Define $g: S^d \rightarrow \mathbb{R}^d$ by $g(w)= \big(F_0(\omega), F_1(\omega), \dots, F_{d-1}(\omega)\big)$.
Since $\mathcal{D}_{X|Y=y,Z=z}$ has a density function with respect to the Lebesgue measure $\mu$ for each $(y,z)\in \mathcal{Y}\times \mathcal{Z}$, $g$ is continuous.
Hence $g(\omega_{x_0})=g(-\omega_{x_0})$ for some $\omega_{x_0}\in S^d$ by the Borsuk-Ulam theorem.
By construction, $g(\omega)+g(-\omega)=(1,\dots, 1)$ for all $\omega\in S^d$.
Thus, $g(\omega_{x_0})=(\frac{1}{2}, \dots, \frac{1}{2})$, which means $h_{\omega_{x_0}}$ is perfectly fair on $\mathcal{D}$.
We just showed that, for any $x\in \mathbb{R}^n$, one can find a perfectly fair linear classifier $h_{\omega_x, x}$ whose decision boundary contains $x$.
Then one can find the countable set $\{h_{\omega_{x_i}, x_i}\}_{i=1}^{\infty}$ whose elements are perfectly fair on $\mathcal{D}$, by using the similar inductive argument made in the case $d=2$.
\end{proof}

\begin{remark}
While Lem.~\ref{lem:3} is stated based on equal opportunity, similar results hold for other fairness metrics; demographic parity and equalized odds.
See Appendix~\ref{sec:C} for details.
\end{remark}

\begin{remark}
In~\cite{Hardt_2016}, Hardt et al. proposed a post-processing method that can find a perfectly fair model on any data distribution.
We note that their method outputs a randomized model, hence it is not applicable to our setting where the hypothesis class consists of deterministic models.
\end{remark}

\begin{figure}[t]
    \centering
    \begin{subfigure}[b]{0.25\columnwidth}
        \includegraphics[width=\columnwidth]{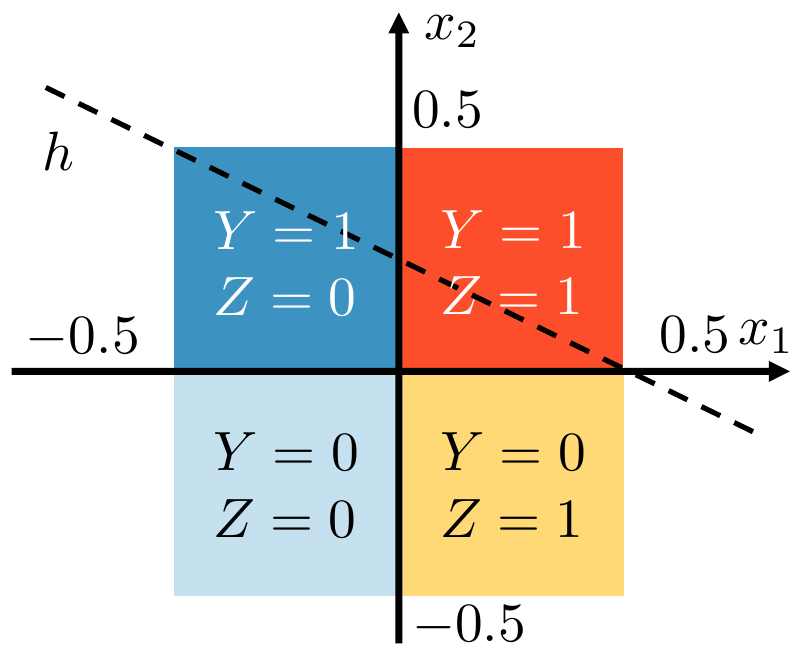}
        \caption{Uncorrupted data $\mathcal{D}$}
        \label{fig:4a}
    \end{subfigure}
    \begin{subfigure}[b]{0.24\columnwidth}
        \includegraphics[width=\columnwidth]{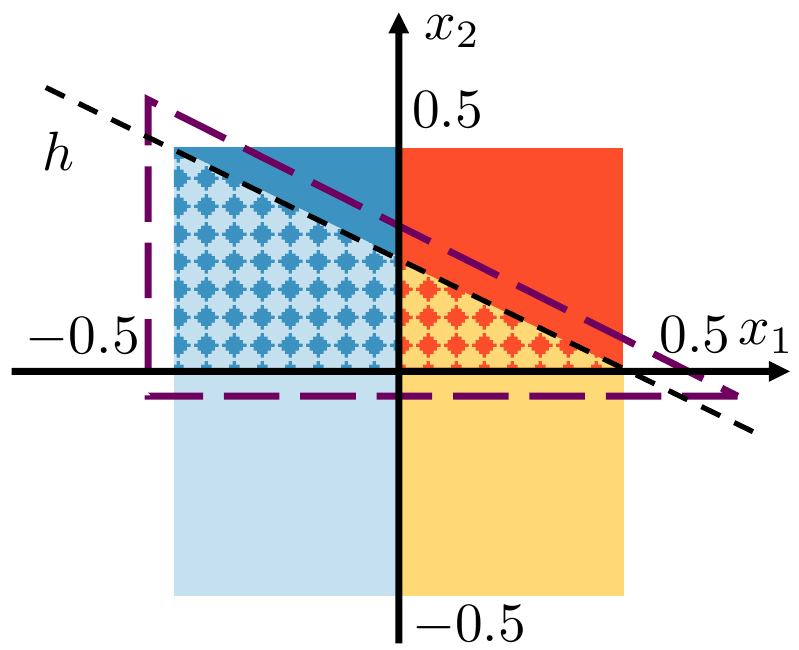}
        \caption{$\widetilde{\mathcal{D}}$}
        \label{fig:4b}
    \end{subfigure}
    \begin{subfigure}[b]{0.24\columnwidth}
        \includegraphics[width=\columnwidth]{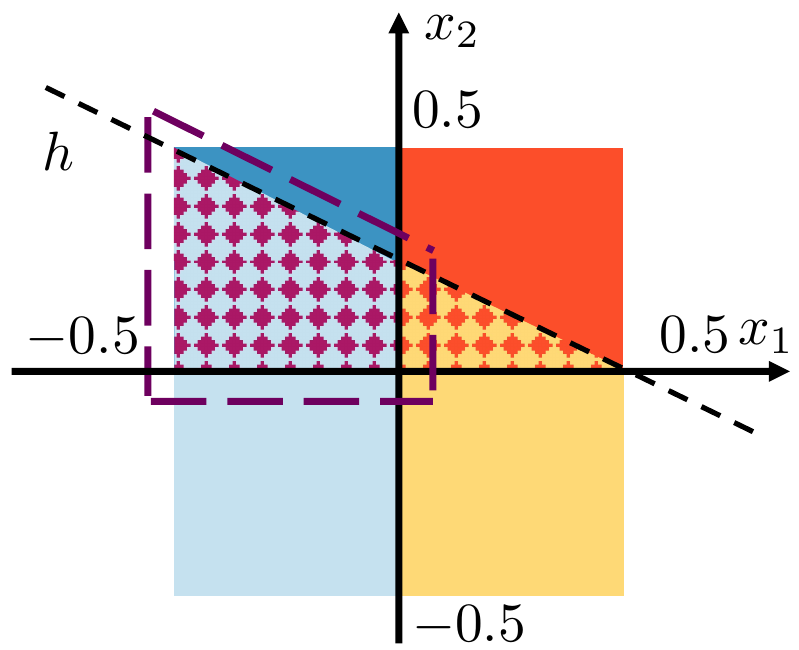}
        \caption{$\text{Fair}_h(\widetilde{\mathcal{D}})$}
        \label{fig:4c}
    \end{subfigure}
    \caption{A visualization of our two-stage attack algorithm with $2$-dimensional feature space $\mathcal{X}$, where $x_1$ and $x_2$ denote the first and second coordinates, respectively.
    (a) The distribution $\mathcal{D}$ follows the description in Example~\ref{ex:1}.
    The target model $h$ predicts samples above its decision boundary (the black dotted line) as positive ($Y=1$).
    (b) In the first stage, the attacker constructs $\widetilde{\mathcal{D}}$ from $\mathcal{D}$ by flipping the $Y$ value with probability $0.6$ (this can be any number in $(0.5,1)$) when $h(X)=0$ and $Y=1$.
    As a result, the triangular region with the dashed boundary is perturbed in the first stage.
    Let $\tilde{f}(x,y,z)$ be the density function of $\widetilde{\mathcal{D}}$.
    For $x=(x_1,x_2)$ in the blue dotted trapezoidal region, $\tilde{f}(x,y,z)$ is $0.4$ if $y=1,z=0$, $0.6$ if $y=0,z=0$, and $0$ otherwise.
    For $x=(x_1,x_2)$ in the red dotted triangular region, $\tilde{f}(x,y,z)$ is $0.4$ if $y=1,z=1$, $0.6$ if $y=0,z=1$, and $0$ otherwise.    
    Then $h$ is the risk minimizer on $\widetilde{\mathcal{D}}$, and $\tilde{p_h}, \tilde{q_h}, \tilde{r_h},\tilde{s_h}$ are $0.075, 0.0625, 0.025, 0.1875$, respectively.
    (c) In the second stage, the attacker constructs $\text{Fair}_h(\widetilde{\mathcal{D}})$ from $\widetilde{\mathcal{D}}$ by flipping the $Z$ value with probability $\sfrac{2}{3}$, computed as per the formula in the second stage, when $h(X)=0, Y=1$, and $Z=0$.
    As a result, the trapezoidal region with the dashed boundary is perturbed in the second stage.
    Let $\tilde{f}_h(x,y,z)$ be the density function of $\text{Fair}_h(\widetilde{\mathcal{D}})$.
    For $x=(x_1,x_2)$ in the purple dotted trapezoidal region, $\tilde{f}_h(x,y,z)$ is $0.133$ if $y=1,z=0$, $0.267$ if $y=1,z=1$, $0.6$ if $y=0,z=0$, and $0$ if $y=0,z=1$.
    }
    \label{fig:4}
\end{figure}

Since $\text{Fair}_h(\mathcal{D})\in \mathcal{S}$ by Lem.~\ref{lem:2}-(i), Lem.~\ref{lem:3} can be applied to $\text{Fair}_h(\mathcal{D})$, and there exist infinitely many linear classifiers that are perfectly fair on $\text{Fair}_h(\mathcal{D})$ (if $n\ge 2$).
If $\mathcal{H}$ contains all linear classifiers (which is usually true), then it includes infinitely many models that are perfectly fair on $\text{Fair}_h(\mathcal{D})$.
Therefore, in general cases, we cannot guarantee that $\mathcal{A}_{0}(\text{Fair}_h(\mathcal{D}))= \{h\}$.

This shows the need of sophisticated attack strategies that guarantee both the minimum risk and the perfect fairness of $h$ on the resulting poisoned distribution.
We now illustrate our two-stage attack strategy that satisfies the desired properties.

\textbf{First stage} The attacker picks any distribution $\widetilde{\mathcal{D}}$ in $\Lambda_{1}(h)$, where $\tilde{f}(x,y,z)$ is the density function of $\widetilde{\mathcal{D}}$.
Recall that $\Lambda_{1}(h)=\{\mathcal{D}' \colon\ \mathcal{D}'\in \mathcal{S}, \mathcal{A}_{1}(\mathcal{D}')= \{h\}\}$, and $\mathcal{A}_{1}(\mathcal{D}')$ is the set of unconstrained risk minimizers on $\mathcal{D}'$.
We note that it is easy to find distributions in $\Lambda_{1}(h)$.
For example, when $\mathcal{H}$ is the set of all measurable functions from $\mathcal{X}$ to $\mathcal{Y}$, $h$ achieves the minimum risk if and only if $h$ is the Bayes classifier on $\widetilde{\mathcal{D}}$.

\textbf{Second stage} The attacker constructs the distribution $\text{Fair}_h(\widetilde{\mathcal{D}})$ with the density function $\tilde{f_h}(x,y,z)$ in a similar way that we get $\text{Fair}_h(\mathcal{D})$ in Sec.~\ref{sec:4:a}.
Specifically, calculating probabilities over $\widetilde{\mathcal{D}}$ in the definition of $p_h,q_h,r_h,s_h$ given by Thm.~\ref{thm:1}, we get $\tilde{p_h}, \tilde{q_h}, \tilde{r_h},\tilde{s_h}$, \emph{e.g.}, $\tilde{p_h} = \Pr_{\widetilde{\mathcal{D}}}(h(X)=0, Y=1, Z=0)$.
Then $\tilde{f_h}(x,y,z)$ is obtained by replacing $p_h,q_h,r_h,s_h,f$ with $\tilde{p_h},\tilde{q_h},\tilde{r_h},\tilde{s_h},\tilde{f}$, respectively, in the construction of $f_h$ in Sec.~\ref{sec:4:a}.
For example, when $\tilde{p_h}+\tilde{r_h}\ge \tilde{q_h}+\tilde{s_h}$ and $\frac{\tilde{q_h}}{\tilde{p_h}}\ge \frac{\tilde{s_h}}{\tilde{r_h}}$, $\tilde{f_h}(x,y,z)$ is equal to $\tilde{f}(x,y,z) + \mathbb{1}(h(x)=1, y=1)\cdot(2z-1)\cdot{\textstyle\frac{\tilde{q_h}\tilde{r_h}-\tilde{p_h}\tilde{s_h}}{(\tilde{p_h}+\tilde{r_h})\tilde{q_h}}} \tilde{f}(x,1,0)$.

Fig.~\ref{fig:4} shows how our two-stage attack algorithm works on a toy example.
The following proposition provides key properties to derive the upper bound on $d^\star_{\text{TV}}(h)$.

\begin{proposition}\label{prop:1}
Let $\widetilde{\mathcal{D}}\in\Lambda_{1}(h)$. Then, (i) $\text{Fair}_h(\widetilde{\mathcal{D}})\in \Lambda_{0}(h)$, and (ii) $d_{\text{TV}}(\widetilde{\mathcal{D}}, \text{Fair}_h(\widetilde{\mathcal{D}}))=C(h, \widetilde{\mathcal{D}})$.
\end{proposition}

\begin{proof}
Using the arguments used in Lem.~\ref{lem:2}, one can similarly show that (a) $\text{Fair}_h(\widetilde{\mathcal{D}}) \in \mathcal{S}$, (b) $h$ is perfectly fair on $\text{Fair}_h(\widetilde{\mathcal{D}})$, and (c) $d_{\text{TV}}(\mathcal{D}, \text{Fair}_h(\widetilde{\mathcal{D}})) = C(h, \widetilde{\mathcal{D}})$.
Then, it suffice to show  $\text{Fair}_h(\widetilde{\mathcal{D}})\in \Lambda_{0}(h)$.
Recall that $\Lambda_{1}(h)=\{\mathcal{D}' \colon\ \mathcal{D}'\in \mathcal{S}, \mathcal{A}_{1}(\mathcal{D}')= \{h\}\}$.
Since $\widetilde{\mathcal{D}}\in\Lambda_{1}(h)$, $h$ is the unique risk minimizer $\arg\min_{g\in \mathcal{H}}  R_{\ell}(g; \widetilde{\mathcal{D}})$.
One can check that any model $g\in \mathcal{H}$ has the same risk on both $\widetilde{\mathcal{D}}$ and $\text{Fair}_h(\widetilde{\mathcal{D}})$, so $h$ is the unique risk minimizer on $\text{Fair}_h(\widetilde{\mathcal{D}})$.
Combining this with the facts that (a) $\text{Fair}_h(\widetilde{\mathcal{D}}) \in \mathcal{S}$ and (b) $h$ is perfectly fair on $\text{Fair}_h(\widetilde{\mathcal{D}})$, one can conclude that $\text{Fair}_h(\widetilde{\mathcal{D}})\in \Lambda_{0}(h)$.
\end{proof}

We are now ready to derive the upper bound on $d^\star_{\text{TV}}(h)$.
The following holds for any $\widetilde{\mathcal{D}}\in \Lambda_{1}(h)$:
\begin{align*}
    d^\star_{\text{TV}}(h) = \inf_{\mathcal{D}'\in \Lambda_{0}(h)} d_{\text{TV}}(\mathcal{D}, \mathcal{D}') \overset{(a)}{\le} d_{\text{TV}}(\mathcal{D}, \text{Fair}_h(\widetilde{\mathcal{D}}))\\ \overset{(b)}{\le} d_{\text{TV}}(\mathcal{D}, \widetilde{\mathcal{D}}) + d_{\text{TV}}(\widetilde{\mathcal{D}}, \text{Fair}_h(\widetilde{\mathcal{D}})) \overset{(c)}{=} d_{\text{TV}}(\mathcal{D}, \widetilde{\mathcal{D}}) + C(h, \widetilde{\mathcal{D}})
\end{align*}
where (a) follows from Prop.~\ref{prop:1}-(i), (b) follows from the triangle inequality, and (c) follows from Prop.~\ref{prop:1}-(ii).
As $d^\star_{\text{TV}}(h) \le d_{\text{TV}}(\mathcal{D}, \widetilde{\mathcal{D}}) + C(h, \widetilde{\mathcal{D}})$ for any $\widetilde{\mathcal{D}}\in \Lambda_{1}(h)$, the upper bound on $d^\star_{\text{TV}}(h)$ in Thm.~\ref{thm:1} can be obtained as follows:
\begin{align}
d^\star_{\text{TV}}(h) \le \inf_{\widetilde{\mathcal{D}}\in \Lambda_{1}(h)} \left(d_{\text{TV}}(\mathcal{D}, \widetilde{\mathcal{D}}) + C(h,
\widetilde{\mathcal{D}})\right).
\end{align}

\section{Sensitive Attribute Flipping Algorithm}\label{sec:5}
We show how the results made in Sec.~\ref{sec:4} can be applied to the design of a computationally efficient flipping attack algorithm against FERM.
When the target model is the unconstrained risk minimizer, as shown in Cor.~\ref{cor:1}, the attack algorithm proposed in Sec.~\ref{sec:4} is optimal.
Indeed, the first stage of the algorithm is not needed at all in this case, and $Z$-flipping in the second stage is sufficient for successful attacks.

\begin{algorithm}[t]
   \caption{$Z$-flipping algorithm}
   \label{alg:1}
\begin{algorithmic}
   \State {\bfseries Input:} The training set $D$, the target model $h_{\text{target}}$.
   \State {\bfseries Output:} The poisoned training set $D'$.
   \For{$(a,b,c)\in \{0,1\}\times \{0,1\} \times \{0, 1\}$}
   \State $D_{a,b,c}\leftarrow \{(x,y,z)\in D\colon\ h_{\text{target}}(x)=a, y=b, z=c \}$
   \EndFor
   \State $P\leftarrow |D_{0,1,0}|$, $Q\leftarrow |D_{1,1,0}|$,
   $R\leftarrow |D_{0,1,1}|$, $S\leftarrow |D_{1,1,1}|$\\
   $\alpha \leftarrow \left\lfloor\sfrac{|PS-QR|}{\max\{P+R, Q+S\}}\right\rfloor$
   \If{$P+R\ge Q+S$ and $\frac{Q}{P}\ge \frac{S}{R}$}
    \State Randomly choose a subset $\mathcal{T}$ of $D_{1,1,0}$ s.t. $|\mathcal{T}|=\alpha$.
    \ElsIf{$P+R\ge Q+S$ and $\frac{Q}{P}< \frac{S}{R}$}
    \State Randomly choose a subset $\mathcal{T}$ of $D_{1,1,1}$ s.t. $|\mathcal{T}|=\alpha$.
    \ElsIf{$P+R< Q+S$ and $\frac{Q}{P}\ge \frac{S}{R}$}
    \State Randomly choose a subset $\mathcal{T}$ of $D_{0,1,1}$ s.t. $|\mathcal{T}|=\alpha$.
    \Else
    \State Randomly choose a subset $\mathcal{T}$ of $D_{0,1,0}$ s.t. $|\mathcal{T}|=\alpha$.
  \EndIf
    \State $\mathcal{T}_{\text{p}} \leftarrow \{(x,y,1-z)\colon\ (x,y,z)\in \mathcal{T}\}$
    \State $D' \leftarrow (D\setminus \mathcal{T})\cup \mathcal{T}_{\text{p}}$
\end{algorithmic}
\end{algorithm}

Inspired by this, we consider the empirical counterpart of the second stage of the attack proposed in Sec.~\ref{sec:4}.
Shown in Alg.~\ref{alg:1} is the pseudocode of our attack algorithm.
In specific, it computes the number of $Z$-flipping, denoted $\alpha$ in Alg.~\ref{alg:1}, using the formula for $C(h, \mathcal{D})$ given in Thm.~\ref{thm:1} where $(p,q,r,s)$ are replaced with the empirical counterparts of them.
Depending on which of the four conditions hold, it chooses a random subset of size $\alpha$ from the corresponding subset of the training set $D$.
It then simply flips the $Z$ values of them to output the poisoned training set $D'$.
The following proposition ensures that Alg.~\ref{alg:1} makes the target model look almost fair on the poisoned training set $D'$ under mild conditions.

\begin{proposition}\label{prop:2}
Let $D=\{(x_i, y_i, z_i)\}_{i=1}^{m}$ be the training set, and $h$ be the target model.
Let $D_{a,b,c}= \{(x,y,z)\in D\colon\ h(x)=a, y=b, z=c \}$, $P=|D_{0,1,0}|$, $Q= |D_{1,1,0}|$, $R=|D_{0,1,1}|$, $S=|D_{1,1,1}|$. If $\frac{P}{m}, \frac{Q}{m}, \frac{R}{m}, \frac{S}{m}$ are $\Omega(1)$, then Alg.~\ref{alg:1} makes $h$ be $O(\frac{1}{m})$-fair on the poisoned training set $D'$.
\end{proposition}

\begin{proof}
We provide the proof for the case where $P+R\ge Q+S$ and $\frac{Q}{P}\ge \frac{S}{R}$, and other cases can be handled in a similar way.
Recall that $\Delta(h, D')$ is equal to
\begin{align*}
     \max_{z\in \{0,1\}} \left| \Pr_{D'}(h(X)=1|Y=1, Z=z) - \Pr_{D'}(h(X)=1|Y=1) \right|,
\end{align*}
and it suffices to show $\Delta(h, D')=O(\frac{1}{m})$.
Let $D'_{a,b,c}= \{(x,y,z)\in D'\colon\ h(x)=a, y=b, z=c \}$.
Since $D'$ is obtained by flipping $Z$ values of samples in $D_{1,1,0}$, we get $|D'_{0,1,0}|=P, |D'_{1,1,0}| = Q-\alpha, |D'_{0,1,1}| = R, |D'_{1,1,1}| = S+\alpha$.
Then, a direct calculation yields
\begin{align}
    \left| \Pr_{D'}(h(X)=1|Y=1, Z=0) - \Pr_{D'}(h(X)=1|Y=1) \right| &= \left| \frac{Q-\alpha}{P+(Q-\alpha)} - \frac{Q+S}{P+Q+R+S} \right|\nonumber \\
    &= \left|\frac{QR-PS-\alpha(P+R)}{(P+Q-\alpha)(P+Q+R+S)} \right|. \label{eq:10}
\end{align}
Since $\alpha = \Big\lfloor\frac{QR-PS}{P+R}\Big\rfloor$, $\frac{QR-PS}{P+R} - \alpha < 1$.
This implies $QR-PS-\alpha(P+R) < P+R$, and \eqref{eq:10} is upper bounded by
\begin{align*}
    & \left|\frac{P+R}{(P+Q-\alpha)(P+Q+R+S)} \right| \le \left|\frac{1}{(P+Q-\alpha)} \right| \le \frac{1}{P},
\end{align*}
where the last inequality comes from $\alpha = \Big\lfloor\frac{QR-PS}{P+R}\Big\rfloor \le \frac{QR-PS}{P+R} \le \frac{QR+PQ}{P+R} = Q$.
Moreover, $\frac{1}{P}=O(\frac{1}{m})$ because $\frac{P}{m}=\Omega(1)$ by the assumption.
Thus, 
$$|\Pr_{D'}(h(X)=1|Y=1, Z=0) - \Pr_{D'}(h(X)=1|Y=1)| = O(\frac{1}{m}).$$
One can similarly show that $$|\Pr_{D'}(h(X)=1|Y=1, Z=1) - \Pr_{D'}(h(X)=1|Y=1)| = O(\frac{1}{m}).$$
Therefore, 
$$\Delta(h, D')=O(\frac{1}{m}).$$
\end{proof}

\begin{figure}[t]
\centering
\includegraphics[width=0.5\columnwidth]{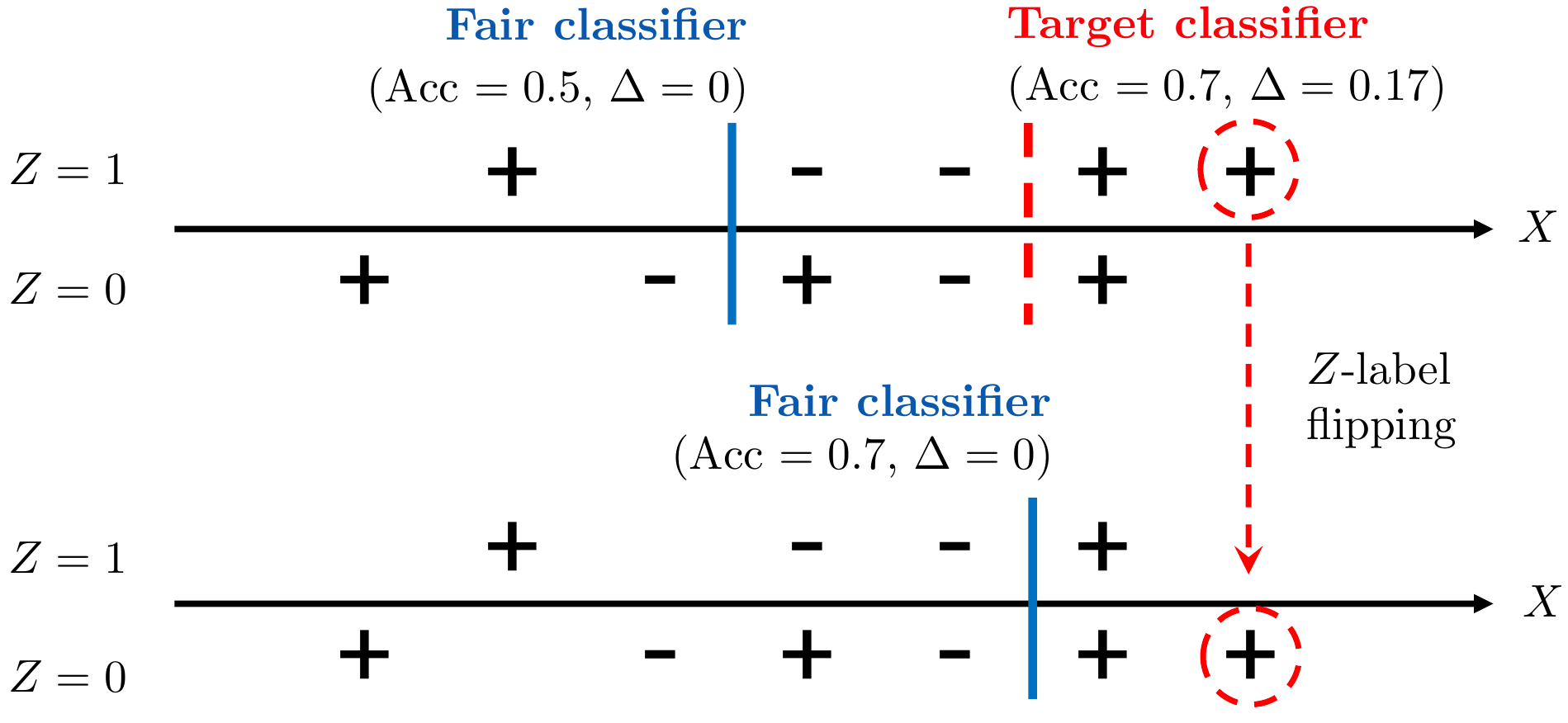}
\caption{A toy example that shows how Alg.~\ref{alg:1} makes FERM output the unconstrained risk minimizer.
We consider a dataset with 10 samples, where $X$ denotes the feature that takes its value in $\mathbb{R}$, $+$ and $-$ denote $Y$ values to be predicted by the learning algorithm, and $Z$ denotes a sensitive attribute (\emph{e.g.,} gender).
We consider linear classifiers that predict samples greater than their thresholds as positive, where the thresholds are shown as vertical lines in the figure.
Acc denotes the accuracy, and $\Delta$ denotes the fairness gap that measures the unfairness of the classifier (see Def.~\ref{def:1} for details).
When $\Delta=0$, the classifier is perfectly fair and satisfies \emph{equal opportunity}~\cite{Hardt_2016}, one of the popular fairness metrics.
On the clean dataset, the fair learning algorithm outputs the fair classifier, the blue solid line, with Acc $=0.5$ and $\Delta=0$.
In this example, the attacker's goal is to make the fair learning algorithm to output the empirical risk minimizer, the red dashed line, which is unfair because $\Delta=0.17$.
By flipping the $Z$ value of the rightmost sample, the attacker can achieve the goal with the minimum number of flipping, thereby degrading the fairness of the fair learning algorithm. }\label{fig:5}
\end{figure}

We focus on the case where the target model $h_\text{target}$ is the empirical risk minimizer on $D$.
Then Alg.~\ref{alg:1} outputs $D'$ on which $h_\text{target}$ looks almost fair by Prop.~\ref{prop:2}.
Moreover, $h_\text{target}$ still achieves the minimum empirical risk on $D'$ because $Z$-flipping does not affect the risk.
Therefore, our attack algorithm increases the chance of $h_\text{target}$ being found by the learner's FERM algorithm, thereby degrading the fairness of FERM.
Fig.~\ref{fig:5} shows how Alg.~\ref{alg:1} makes FERM output the empirical risk minimizer with a toy example.

\section{Experimental Results}\label{sec:6}
We generate a synthetic dataset $D=\{(x_i, y_i, z_i)\}_{i=1}^{6000}$ where $x_i\in \mathbb{R}^2, y_i\in \{0, 1\}, z_i\in \{0,1\}$ for $1\le i \le 6000$, following the method used in~\cite{Zafar2017FC}.
Specifically, we set $y_i=1$ for $1 \le i \le 3000$ and $y_i=0$ for $3001 \le i \le 6000$.
Then we randomly draw $x_1,\dots, x_{3000}$ from $\mathcal{N}_1$ and $x_{3001},\dots, x_{6000}$ from $\mathcal{N}_2$, where $\mathcal{N}_1$ and $\mathcal{N}_2$ are Gaussian distributions $N([2;2], [5,1; 1,5])$ and $N([-2;-2], [10,1; 1,3])$, respectively.
Let $f_1$ and $f_2$ be the density functions of $\mathcal{N}_1$ and $\mathcal{N}_2$, respectively.
For $1\le i \le 6000$, we draw $z_i$ from the Bernoulli distribution $\text{Bern}(\frac{f_1(x_i')}{f_1(x_i')+f_2(x_i')})$ where $x_i'=[\cos\frac{\pi}{6}, -\sin\frac{\pi}{6}; \sin\frac{\pi}{6}, \cos\frac{\pi}{6}]x_i$.
Note that $y_i$ and $z_i$ have a correlation by construction, so the solution of vanilla ERM will be unfair.
Let $D_{\text{train}}$ and $D_{\text{test}}$ denote the training set and the test set, respectively.
All experiments are repeated 5 times, and the accuracy and unfairness are measured on $D_{\text{test}}$; we use $\Delta(h_{\text{target}}, D_{\text{test}})$ to quantify the unfairness of $h_{\text{target}}$.

We compare our attack algorithm with data poisoning attack algorithms: (1) random $Y$-flip chooses random samples from $D_{\text{train}}$ and flips $Y$ values; (2) random $Z$-flip chooses random samples and flips $Z$ values; (3) random $Y\&Z$-flip chooses random samples and flips both $Y$ and $Z$ values; (4) adversarial sampling (AS) chooses adversarial samples from the feasible attack set using the online gradient descent algorithm proposed in~\cite{Chang2020Adversarial} and adds them to $D_{\text{train}}$.
We evaluate these attacks against fair learning algorithms: (1) in-processing method using fairness constrains (FC)~\cite{Zafar2017FC}; (2) fair training against adversarial perturbations (Err-Tol)~\cite{celis2021fair}; (3) fair and robust training (FR-Train)~\cite{Roh2020FR-Train} given the clean validation set.

\begin{table}[t]
\caption{Comparison with other baseline attack algorithms.
The fairness gap $\Delta$ measures the unfairness of the model.
The target model $h_{\text{target}}$ is the output of logistic regression; the accuracy and fairness gap are 0.88 and 0.19, respectively.
Our $Z$-flip attack makes the output be significantly unfair, with only 3.2\% of poisoning rate.
}\label{table:1}
\begin{center}
\begin{small}
\begin{tabular}{lcccccccc}
\toprule
 & \multicolumn{2}{c}{FC~\cite{Zafar2017FC}}   & \multicolumn{2}{c}{Err-Tol~\cite{celis2021fair}} & \multicolumn{2}{c}{FR-Train~\cite{Roh2020FR-Train}} \\
\cmidrule{2-7}
Attack method & Acc. & $\Delta$ & Acc. & $\Delta$ & Acc. & $\Delta$ \\
\midrule
Uncorrupted  & 0.79 & 0.05 & 0.81  & 0.06 & 0.79 & 0.03 \\
Random Y-flip  & 0.77 & 0.01 & 0.75 & 0.03 & 0.76 & 0.02 \\
Random Z-flip  & 0.79 & 0.06 & 0.87 & 0.18 & 0.81 & 0.04 \\
Random Y\&Z-flip  & 0.80 & 0.07 & 0.88 & 0.19 & 0.78  & 0.03 \\
AS~\cite{Chang2020Adversarial}  & 0.78 & 0.02 & 0.78 & 0.03 & 0.77  & 0.02 \\
Our Z-flip  & 0.85 & \textbf{0.14}  & 0.88 & \textbf{0.19} & 0.82 & \textbf{0.08} \\
\bottomrule
\end{tabular}
\end{small}
\end{center}
\end{table}

We find $h_{\text{target}}$ via empirical risk minimization with logistic loss and get the poisoned training set $D'$ using Alg.~\ref{alg:1}.
Shown in Table~\ref{table:1} is the performance of attack algorithms against fair learning algorithms.
When the learner runs fair learning algorithms on the uncorrupted dataset, the fairness gap significantly decreases at the cost of degraded accuracy, exhibiting a well-known tradeoff between accuracy and fairness.
However, with only $3.2\%$ of poisoning rate, our $Z$-flip attack makes the output be significantly unfair, outperforming (or achieving comparable attack performances to) other attack baselines.
Interestingly, our attack successfully degrades the fairness of robust fair training algorithms; Err-Tol and FR-Train.
Err-Tol essentially achieves its robustness by relaxing the fairness threshold of its constraints, where the relaxed threshold is carefully calculated using the known poisoning rate.
By Prop.~\ref{prop:2}, our attack makes $h_{\text{target}}$ look almost fair on $D'$, so $h_{\text{target}}$ satisfies the fairness constraint of Err-Tol.
As $h_{\text{target}}$ still minimizes the empirical risk on $D'$, Err-Tol will output the model close to $h_{\text{target}}$.
FR-Train makes use of the clean validation set to achieve the robustness, but its performance on adversarial $Z$-flip attacks is not studied in the previous work.
We empirically show that our $Z$-flip attack makes FR-Train output an unfair model with the fairness gap of $0.08$.

\section{Conclusion}\label{sec:7}
We studied poisoning attacks against risk minimization with fairness constraints.
We found the lower and upper bounds on the minimum amount of data perturbation required for successful flipping attack for the case of true risk minimization with fairness constraints. 
Inspired by the fact that sensitive attribute flipping attack is optimal for certain cases, we designed an efficient $Z$-flipping attack algorithm that can compromise the performance of Fair Empirical Risk Minimization (FERM).
We empirically showed that our attack algorithm can degrade the fairness of FERM on synthetic data against existing fair learning algorithms.

We conclude our paper by enumerating important open problems.
Our attack algorithm is optimal and our bounds are tight when the target model is the unique unconstrained risk minimizer.
Tightening the lower and upper bounds in Thm.~\ref{thm:1} for a general target model is an important future work.
Our theoretical analysis is limited to the case where both $\mathcal{Y}$ and $\mathcal{Z}$ are binary.
We conjecture the theoretical analysis can be extended to the case where $\mathcal{Y}$ and $\mathcal{Z}$ are non-binary.
Moreover, it would be interesting to extend our attack algorithm into the federated learning setting.

\section{Acknowledgements}
This work was supported in part by NSF Award DMS-2023239, NSF/Intel Partnership on Machine Learning for Wireless Networking Program under Grant No. CNS-2003129, and the Understanding and Reducing Inequalities Initiative of the University of Wisconsin-Madison, Office of the Vice Chancellor for Research and Graduate Education with funding from the Wisconsin Alumni Research Foundation.

\bibliographystyle{plainnat}
\bibliography{poisoningattack_bib}

\newpage
\appendix

\section{Extension to Another Fairness Metric: Demographic Parity~\cite{Feldman_2015}}\label{sec:A}
We can measure the fairness gap with respect to demographic parity (DP) as follows.

\begin{definition}\label{def:2} The \emph{fairness gap}, measured with respect to demographic parity, of a model $h: \mathcal{X}\rightarrow \mathcal{Y}$ on the distribution $\mathcal{D}$, denoted $\Delta_{\text{DP}}(h, \mathcal{D})$, is
$$\max_{z\in \mathcal{Z}} \Big|\Pr_{\mathcal{D}} (h(X)=1|Z=z) - \Pr_{\mathcal{D}} (h(X)=1) \Big|.$$
For $\delta \in [0, 1]$, the model $h$ is \emph{$\delta$-fair w.r.t. DP} on $\mathcal{D}$ if $\Delta_{\text{DP}}(h, \mathcal{D}) \le \delta$.
The model $h$ is \emph{perfectly fair w.r.t. DP} on $\mathcal{D}$ if it is $0$-fair.
We similarly define the fairness gap, $\delta$-fairness, and perfect fairness of $h$ on the training set $D$ by using the empirical probability $\Pr_D(\cdot)$ over $D$.
\end{definition}

We now formulate our problem with the fairness gap with respect to DP. The learner solves the following constrained optimization problem:
\begin{equation}
  \min_h \{R_\ell(h; \mathcal{D}) \colon\ h\in \mathcal{H}, h \text{~is perfectly fair w.r.t. DP on~}\mathcal{D}\}.
\end{equation}
The attacker solves the following bilevel optimization problem:
\begin{align}
\min_{\mathcal{D}'} \big\{d_{\text{TV}}(\mathcal{D}, \mathcal{D}') \colon\ \mathcal{D}'\in \mathcal{S}, \mathcal{A}_{0}(\mathcal{D}')= \{h_{\text{target}}\} \big\}
\end{align}
where $\mathcal{A}_{\delta}(\mathcal{D})$ is the set of solutions of $\min_h\{R_{\ell}(h;\mathcal{D}) \colon\ h\in \mathcal{H},  h \text{~is $\delta$-fair w.r.t. DP on $\mathcal{D}$}\}$, and the search space is defined as per \eqref{eq:2}.
Observe that the only difference between Def.~\ref{def:1} and Def.~\ref{def:2} is that the probability is not conditioned on $Y=1$ in Def.~\ref{def:2}.
Thus all the arguments made in Thm.~\ref{thm:1} still hold with proper adjustments.
In specific, we let $p_h={\textstyle\Pr_{\mathcal{D}}}(h(X)=0, Z=0), q_h= {\textstyle\Pr_{\mathcal{D}}}(h(X)=1, Z=0), r_h={\textstyle\Pr_{\mathcal{D}}}(h(X)=0, Z=1), s_h= {\textstyle\Pr_{\mathcal{D}}}(h(X)=1, Z=1), \tilde{p_h}={\textstyle\Pr_{\widetilde{\mathcal{D}}}}(h(X)=0, Z=0), \tilde{q_h}= {\textstyle\Pr_{\widetilde{\mathcal{D}}}}(h(X)=1, Z=0), \tilde{r_h}={\textstyle\Pr_{\widetilde{\mathcal{D}}}}(h(X)=0, Z=1), \tilde{s_h}= {\textstyle\Pr_{\widetilde{\mathcal{D}}}}(h(X)=1, Z=1)$.
Then we get the following theorem.

\begin{theorem}\label{thm:2} Let $h$ be any model in the hypothesis class $\mathcal{H}$.
Then, $C(h, \mathcal{D}) \le d^\star_{\text{TV}}(h) \le {\textstyle\inf_{\widetilde{\mathcal{D}}\in \Lambda_{1}(h)}} (d_{\text{TV}}(\mathcal{D}, \widetilde{\mathcal{D}}) + C(h,  \widetilde{\mathcal{D}}))$ where $C(h, \mathcal{D}) = \frac{|p_hs_h-q_hr_h|}{\max\{p_h+r_h, q_h+s_h\}}$ and $C(h,  \widetilde{\mathcal{D}}) = \frac{|\tilde{p_h}\tilde{s_h}-\tilde{q_h}\tilde{r_h}|}{\max\{\tilde{p_h}+\tilde{r_h}, \tilde{q_h}+\tilde{s_h}\}}$.
\end{theorem}

We also show how to construct the distribution $\text{DPFair}_h(\mathcal{D})$ that matches the lower bound on $d^\star_{\text{TV}}(h)$.
Let $f$ be the density function of $\mathcal{D}$.
We construct $\text{DPFair}_h(\mathcal{D})$ with the density function $f_h(x,y,z)$ defined as follows.

\noindent\textbf{Case 1.} $p_h+r_h\ge q_h+s_h$, $\frac{q_h}{p_h}\ge \frac{s_h}{r_h}$: Define

$$f_h(x,y,z):= f(x,y,z) + \mathbb{1}(h(x)=1)\cdot(2z-1)\cdot\frac{q_hr_h-p_hs_h}{(p_h+r_h)q_h} f(x,y,0). $$
\textbf{Case 2.} $p_h+r_h\ge q_h+s_h$, $\frac{q_h}{p_h}<\frac{s_h}{r_h}$: Define 

$$f_h(x,y,z):=f(x,y,z) + \mathbb{1}(h(x)=1)\cdot(1-2z)\cdot \frac{p_hs_h-q_hr_h}{(p_h+r_h)s_h} f(x,y,1).$$
\textbf{Case 3.} $p_h+r_h< q_h+s_h$, $\frac{q_h}{p_h}\ge \frac{s_h}{r_h}$: Define
$$f_h(x,y,z):=f(x,y,z) + \mathbb{1}(h(x) \neq 1)\cdot(1-2z)\cdot \frac{q_hr_h-p_hs_h}{(q_h+s_h)r_h} f(x,y,1).$$
\textbf{Case 4.} $p_h+r_h< q_h+s_h$, $\frac{q_h}{p_h}< \frac{s_h}{r_h}$: Define
$$f_h(x,y,z):=f(x,y,z) + \mathbb{1}(h(x) \neq 1)\cdot(2z-1)\cdot \frac{p_hs_h-q_hr_h}{(q_h+s_h)p_h} f(x,y,0).$$

\section{Connection with Theorem 1 in~\cite{Wang2020Robust}}\label{sec:B}
We continue from Remark~\ref{rmk:1}.
Let $\mathcal{X}= \mathbb{R}^n$, $\mathcal{Y}=\{0,1\}$, $\mathcal{Z}=\{0,1\}$.
Let $\mathcal{D}$ be a probability distribution over $\mathcal{X}\times\mathcal{Y}\times\mathcal{Z}$ with the density function $f(x,y,z)$.
We consider a noisy distribution $\mathcal{D}'$ with the density function $f'(x,y,z)$.
In~\cite{Wang2020Robust}, Wang et al. assume that $\mathcal{D}'_{(X,Y)}$ is equal to $\mathcal{D}_{(X,Y)}$, \emph{i.e.,} $\sum_{z\in \{0,1\}} f(x,y,z) = \sum_{z\in \{0,1\}} f'(x,y,z)$.
Then Theorem 1 in~\cite{Wang2020Robust} implies that if a model $h$ is perfectly fair w.r.t. DP on the distribution $\mathcal{D}'$, then $d_{\text{TV}}(\mathcal{D}_{Z=z}, \mathcal{D}'_{Z=z}) \ge \left|\Pr_{\mathcal{D}}(h(X)=1| Z=z) - \Pr_{\mathcal{D}}(h(X)=1) \right|$ for each $z\in\{0,1\}$.
However, they did not provide an explicit construction of $\mathcal{D}'$ that matches the bound.
We show that our construction scheme $\text{DPFair}_h(\mathcal{D})$ defined in Appendix~\ref{sec:A} matches the lower bound for certain cases.

Let $p_h={\textstyle\Pr_{\mathcal{D}}}(h(X)=0, Z=0), q_h= {\textstyle\Pr_{\mathcal{D}}}(h(X)=1, Z=0), r_h={\textstyle\Pr_{\mathcal{D}}}(h(X)=0, Z=1), s_h= {\textstyle\Pr_{\mathcal{D}}}(h(X)=1, Z=1)$, $A=\{x\in \mathbb{R}^n\colon\ h(x)=1\}$.
If we consider the case where $p_h+r_h\ge q_h+s_h$ and $\frac{q_h}{p_h}\ge \frac{s_h}{r_h}$, the density function $f_h(x,y,z)$ of $\text{DPFair}_h(\mathcal{D})$ can be computed as follows: $f(x,y,z) - \frac{q_hr_h-p_hs_h}{(p_h+r_h)q_h} f(x,y,0)$ if $h(x)=1, z=0$; $f(x,y,z) + \frac{q_hr_h-p_hs_h}{(p_h+r_h)q_h} f(x,y,0)$ if $h(x)=1, z=1$; $f(x,y,z)$ otherwise.

A direct calculation yields $\Pr_{\text{DPFair}_h(\mathcal{D})}(h(X)=0, Z=0) = \sum_{y\in \{0,1\}} \int_{A^c} f_h(x,y,0) \diff{\mu} = \sum_{y\in \{0,1\}} \int_{A^c} f(x,y,0) \diff{\mu} =p_h$ and $\Pr_{\text{DPFair}_h(\mathcal{D})}(h(X)=1, Z=0) = \sum_{y\in \{0,1\}} \int_A f_h(x,y,0) \diff{\mu} = \frac{p_hq_h+p_hs_h}{(p_h+r_h)q_h} \sum_{y\in \{0,1\}} \int_A f(x,y,0) \diff{\mu} = \frac{p_hq_h+p_hs_h}{p_h+r_h}$.
Similarly, we get $\Pr_{\text{DPFair}_h(\mathcal{D})}(h(X)=0, Z=1) = r_h$ and $\Pr_{\text{DPFair}_h(\mathcal{D})}(h(X)=1, Z=1) = \frac{r_hq_h+r_hs_h}{p_h+r_h}$.
Since $p_h+q_h+r_h+s_h=1$, one can get $\Pr_{\text{DPFair}_h(\mathcal{D})}(Z=0) =  \frac{p_h}{p_h+r_h}$ and $\Pr_{\text{DPFair}_h(\mathcal{D})}(Z=1) =  \frac{r_h}{p_h+r_h}$.

Then $\mathcal{D}_{Z=0}$ and $\text{DPFair}_h(\mathcal{D})_{Z=0}$ have the following density functions
$$ f(x,y|z=0) = \begin{cases} \frac{1}{p_h+q_h} f(x,y,0) ~~\text{if}~~h(x)=1\\
\frac{1}{p_h+q_h} f(x,y,0)~~\text{if}~~h(x) \neq 1
\end{cases}$$
and
$$f_h(x,y|z=0) = \begin{cases} \frac{q_h+s_h}{q_h} f(x,y,0) ~~\text{if}~~h(x)=1\\
\frac{p_h+r_h}{p_h} f(x,y,0)~~\text{if}~~h(x) \neq 1
\end{cases},$$
respectively.
One can check the following.
\begingroup
\allowdisplaybreaks
\begin{align*}
    &d_{\text{TV}}(\mathcal{D}_{Z=0}, \text{DPFair}_h(\mathcal{D})_{Z=0})\\
    &= \frac{1}{2} \sum_{y\in \{0,1\}} \int_{\mathbb{R}^n}\left|f(x,y|z=0)-f_h(x,y|z=0)\right| \diff\mu\\
    &= \frac{1}{2} \sum_{y\in \{0,1\}} \int_{A}\left|f(x,y|z=0)-f_h(x,y|z=0)\right| \diff\mu + \frac{1}{2} \sum_{y\in \{0,1\}} \int_{A^c}\left|f(x,y|z=0)-f_h(x,y|z=0)\right| \diff\mu\\
    &= \frac{1}{2} \sum_{y\in \{0,1\}} \int_{A}\left|\frac{1}{p_h+q_h} f(x,y,0)-\frac{q_h+s_h}{q_h} f(x,y,0)\right| \diff\mu + \frac{1}{2} \sum_{y\in \{0,1\}} \int_{A^c}\left|\frac{1}{p_h+q_h} f(x,y,0)-\frac{p_h+r_h}{p_h} f(x,y,0)\right| \diff\mu\\
    &= \frac{1}{2} \left|\frac{1}{p_h+q_h} -\frac{q_h+s_h}{q_h}\right| \sum_{y\in \{0,1\}} \int_{A} f(x,y,0) \diff\mu + \frac{1}{2} \left|\frac{1}{p_h+q_h}-\frac{p_h+r_h}{p_h}\right|\sum_{y\in \{0,1\}} \int_{A^c} f(x,y,0) \diff\mu\\
    &= \frac{1}{2} \left|\frac{p_h+q_h+r_h+s_h}{p_h+q_h} -\frac{q_h+s_h}{q_h}\right| q_h + \frac{1}{2} \left|\frac{p_h+q_h+r_h+s_h}{p_h+q_h}-\frac{p_h+r_h}{p_h}\right|p_h\\
    &=\frac{q_hr_h-p_hs_h}{p_h+q_h} = \left|\frac{q_h}{p_h+q_h} - \frac{q_h+s_h}{p_h+q_h+r_h+s_h} \right| \\
    &= \left|\Pr_{\mathcal{D}}(h(X)=1| Z=0) - \Pr_{\mathcal{D}}(h(X)=1) \right|.
\end{align*}
\endgroup
Hence $\text{DPFair}_h(\mathcal{D})$ matches the lower bound for $z=0$. Similarly, $\mathcal{D}_{Z=1}$ and $\text{DPFair}_h(\mathcal{D})_{Z=1}$ have the following joint density functions
$$ f(x,y|z=1) = \begin{cases} \frac{1}{r_h+s_h} f(x,y,1) ~~\text{if}~~h(x)=1\\
\frac{1}{r_h+s_h} f(x,y,1)~~\text{if}~~h(x) \neq 1
\end{cases}$$
and
\begin{align*}
&f_h(x,y|z=1)\\
&= \begin{cases} \frac{p_h+r_h}{r_h} f(x,y,1) + \frac{q_hr_h-p_hs_h}{q_hr_h} f(x,y,0) ~~\text{if}~~h(x)=1\\
\frac{p_h+r_h}{r_h} f(x,y,1)~~\text{if}~~h(x) \neq 1
\end{cases},
\end{align*}
respectively.
One can check the following.
\begingroup
\allowdisplaybreaks
\begin{align*}
    &d_{\text{TV}}(\mathcal{D}_{Z=1}, \text{DPFair}_h(\mathcal{D})_{Z=1})\\
    &= \frac{1}{2} \sum_{y\in \{0,1\}} \int_{\mathbb{R}^n}\left|f(x,y|z=1)-f_h(x,y|z=1)\right| \diff\mu\\
    &= \frac{1}{2} \sum_{y\in \{0,1\}} \int_{A}\left|f(x,y|z=1)-f_h(x,y|z=1)\right| \diff\mu + \frac{1}{2} \sum_{y\in \{0,1\}} \int_{A^c}\left|f(x,y|z=1)-f_h(x,y|z=1)\right| \diff\mu\\
    &\overset{(a)}{\ge}  \frac{1}{2} \sum_{y\in \{0,1\}} \left|\int_{A}f(x,y|z=1)-f_h(x,y|z=1) \diff\mu \right| + \frac{1}{2} \sum_{y\in \{0,1\}} \int_{A^c}\left|f(x,y|z=1)-f_h(x,y|z=1)\right| \diff\mu\\
    &=\frac{q_hr_h-p_hs_h}{r_h+s_h} = \left|\frac{s_h}{r_h+s_h} - \frac{q_h+s_h}{p_h+q_h+r_h+s_h} \right| \\
    &= \left|\Pr_{\mathcal{D}}(h(X)=1| Z=1) - \Pr_{\mathcal{D}}(h(X)=1) \right|,
\end{align*}
\endgroup
where (a) comes from the triangle inequality.
Hence $\text{DPFair}_h(\mathcal{D})$ matches the lower bound for $z=1$ if (a) is the equality.
The equality condition of (a) depends on the behavior of the density function $f(x,y,z)$ for $x\in A$.

\section{Generalization of Lemma~\ref{lem:3}}\label{sec:C}

We now extend Lem.~\ref{lem:3} to other fairness criteria such as demographic parity~\cite{Feldman_2015} and equalized odds~\cite{Hardt_2016}.

\begin{definition}\label{def:3}
A model $h: \mathcal{X}\rightarrow \mathcal{Y}$ satisfies \emph{demographic parity} on $\mathcal{D}$ if, for all $z\in \mathcal{Z}$,
$$\Pr_{\mathcal{D}} (h(X)=1|Z=z)=\Pr_{\mathcal{D}} (h(X)=1).$$
A model $h: \mathcal{X}\rightarrow \mathcal{Y}$ satisfies \emph{equal opportunity} on $\mathcal{D}$ if, for all $z\in \mathcal{Z}$,
$$\Pr_{\mathcal{D}} (h(X)=1|Y=1, Z=z)=\Pr_{\mathcal{D}} (h(X)=1|Y=1).$$
A model $h: \mathcal{X}\rightarrow \mathcal{Y}$ satisfies \emph{equalized odds} on $\mathcal{D}$ if, for all $y\in\mathcal{Y}, z\in \mathcal{Z}$,
$$\Pr_{\mathcal{D}} (h(X)=1|Y=y, Z=z)=\Pr_{\mathcal{D}} (h(X)=1|Y=y).$$
\end{definition}

The following lemma shows the existence of infinitely many linear classifiers satisfying fairness criteria defined in Def.~\ref{def:3}.

\begin{lemma} Let $\mathcal{X}= \mathbb{R}^n$, $\mathcal{Y}=\{0,1\}$, $\mathcal{Z}=\{0,1, \dots, d-1\}$.
Let $\mathcal{D}$ be a probability distribution over $\mathcal{X}\times\mathcal{Y}\times\mathcal{Z}$ whose conditional distribution $\mathcal{D}_{X|Y=y,Z=z}$ has a density function with respect to the Lebesgue measure $\mu$ for each $(y,z)\in \mathcal{Y}\times \mathcal{Z}$.
The following hold.
(i) If $n \geq d+\mathbb{1}(d\geq 3)$, then there exist infinitely many linear classifiers that satisfy demographic parity on $\mathcal{D}$.
Among such linear classifiers, for any $x\in \mathcal{X}$, there exist at least one linear classifier whose decision boundary passes through $x$.
(ii) Exactly the same statement holds for equal opportunity.
(iii) If $n \geq 2d+1$, then there exist infinitely many linear classifiers that satisfy equalized odds on $\mathcal{D}$.
Among such linear classifiers, for any $x\in \mathcal{X}$, there exist at least one linear classifier whose decision boundary passes through $x$.
\end{lemma}

\begin{proof}
(ii) is the result of Lem.~\ref{lem:3}.
Moreover, (i) can be handled with similar arguments made in Lem.~\ref{lem:3}.
Specifically, in Case 2 of the proof for Lem.~\ref{lem:3}, one can consider $F_i(\omega):=\Pr(h_{\omega, x_0}(X)=1| Z=i)$ instead of $\Pr(h_{\omega, x_0}(X)=1|Y=1, Z=i)$.
With this modification, one can easily get the desired result.

We now prove (iii). Suppose $n\ge 2d+1$.
Let $S^{2d}=\{x\in \mathbb{R}^{2d+1}:\|x\|=1\}$.
Define the natural embedding $\iota: S^{2d} \rightarrow \mathbb{R}^n$ by $\iota\big((x_1, \dots, x_{2d+1})\big)=(x_1, \dots, x_{2d+1}, 0, \dots, 0)$.
For any $x_0\in \mathbb{R}^n$, consider the following linear classifiers parametrized by $\omega \in S^{2d}$; $$h_{\omega, x_0}(x):=\begin{cases} 1 \text{~~if~~} (x-x_0)^T\cdot \iota (\omega) \ge 0\\
0 \text{~~o.w.}\end{cases}.$$
Let $F_i(\omega):=\Pr(h_{\omega, x_0}(X)=1|Y=1, Z=i)$ and $G_i(\omega):=\Pr(h_{\omega, x_0}(X)=1|Y=0, Z=i)$ for $0\le i \le d-1$.
Define $g: S^{2d} \rightarrow \mathbb{R}^{2d}$ by $g(w)= \big(F_0(\omega), F_1(\omega), \dots, F_{d-1}(\omega), G_0(\omega), G_1(\omega), \dots, G_{d-1}(\omega)\big)$.
Since each conditional distribution has a density function with respect to the Lebesgue measure $\mu$, $g$ is continuous.
Hence $g(\omega_{x_0})=g(-\omega_{x_0})$ for some $\omega_{x_0}\in S^{2d}$ by the Borsuk-Ulam theorem. By construction, $g(\omega)+g(-\omega)=(1,\dots, 1)$ for all $\omega\in S^{2d}$.
Therefore, $g(\omega_{x_0})=(\frac{1}{2}, \dots, \frac{1}{2})$, which means $h_{\omega_{x_0}}$ is perfectly fair on $\mathcal{D}$.
We just showed that, for any $x\in \mathbb{R}^n$, one can find a perfectly fair (w.r.t. equalized odds) linear classifier $h_{\omega_x, x}$ whose decision boundary passes $x$.
Then one can find the countable set $\{h_{\omega_{x_i}, x_i}\}_{i=1}^{\infty}$ whose elements satisfy equalized odds on $\mathcal{D}$, by using the similar inductive argument made in Lem.~\ref{lem:3}.
\end{proof}

\end{document}